\documentclass[11pt]{article}
\usepackage{amssymb, amsmath, amsthm}

\usepackage[margin=1in]{geometry}
\usepackage{amsfonts,mathtools,mathrsfs,mathtools,color}
\usepackage[colorlinks]{hyperref}
\usepackage{graphicx}
\usepackage[table]{xcolor}
\usepackage{tocloft}
\usepackage{bigints}
\usepackage{makecell}
\usepackage{algorithm}
\usepackage{algorithmic}
\usepackage{tikz}
\usepackage{pgfplots}
\pgfplotsset{compat=1.18}
\usepgfplotslibrary{groupplots}
\usepackage{centernot}
\usepackage{bm}
\usepackage{nicefrac}
\usepackage{dsfont}
\usepackage{enumitem}
\usepackage{booktabs}
\usepackage{xspace}
\usepackage{stackengine,scalerel}
\usepackage{multirow}

\usepackage[nameinlink,capitalise]{cleveref}

\usepackage{caption}
\usetikzlibrary{math}
\usetikzlibrary{decorations.markings, calc, arrows} 
\usetikzlibrary{arrows.meta}

\definecolor{coolblack}{rgb}{0.0, 0.0, 0.230}
\hypersetup{
    colorlinks,
    linkcolor={red!70!black},
    citecolor={blue!75!black},
    urlcolor={blue!75!black}
}

\newtheorem{proposition}{Proposition}
\newtheorem{theorem}{Theorem}
\newtheorem*{theorem*}{Theorem}
\newtheorem{lemma}{Lemma}
\newtheorem*{lemma*}{Lemma}
\newtheorem{corollary}{Corollary}

\newtheorem{assumption}{Assumption}

\makeatletter
\newenvironment{proofnopunct}[1][\proofname]{%
  \par
  \pushQED{\qed}%
  \normalfont
  \topsep6\p@\@plus6\p@\relax
  \trivlist
  \item[\hskip\labelsep\itshape #1]\ignorespaces
}{%
  \popQED\endtrivlist\@endpefalse
}
\makeatother

\newenvironment{proofof}[1]
  {\begin{proofnopunct}[Proof of #1]}
  {\end{proofnopunct}}

\usepackage{enumitem}

\newlist{reglist}{enumerate}{1}
\setlist[reglist]{
    label=\textup{(R\arabic*)},
    ref=R\arabic*,
    leftmargin=*,
    itemsep=0.8em
}

\crefformat{reglisti}{#2#1#3}
\Crefformat{reglisti}{#2#1#3}

\usepackage{romanbar}

\newcommand{\op}{\mathsf{op}}
\newcommand{\dvert}{\mathbin{\|}}

\renewcommand{\P}{\mathcal{P}}
\newcommand{\R}{\mathbb{R}} 
\newcommand{\N}{\mathcal{N}}

\newcommand{\E}{\mathbb{E}}

\newcommand{\bbN}{\mathbb{N}}

\newcommand{\Tr}{\mathsf{Tr}}

\newcommand{\Wass}{\mathsf{W}}

\newcommand{\KL}{\mathsf{KL}}

\newcommand{\FI}{\mathsf{FI}}

\newcommand{\sfW}{\mathsf{W}}

\newcommand{\deq}{\coloneqq}

\renewcommand{\d}{\mathrm{d}}

\newcommand{\ds}{\,\mathrm{d}s}
\newcommand{\dt}{\,\mathrm{d}t}
\newcommand{\du}{\,\mathrm{d}u}
\newcommand{\dx}{\,\mathrm{d}x}
\newcommand{\dy}{\,\mathrm{d}y}

\newcommand{\dz}{\,\mathrm{d}z}

\newcommand{\I}{\mathrm{I}}

\newcommand{\ac}{\mathrm{ac}}

\newcommand{\fs}{\mathrm{fs}}

\newcommand{\PtwoacfsRd}{\mathcal{P}_{2, \ac, \fs}(\R^d)}

\makeatletter
\newtheorem*{rep@theorem}{\rep@title}
\newcommand{\newreptheorem}[2]{%
\newenvironment{rep#1}[1]{%
 \def\rep@title{#2 \ref{##1}}%
 \begin{rep@theorem}}%
 {\end{rep@theorem}}}
\makeatother

\newreptheorem{theorem}{Theorem}
\newreptheorem{lemma}{Lemma}

\newcommand{\RHMC}{\nameref{alg:RHMC}\xspace}

\newcommand{\HMC}{\textsf{HMC}\xspace}
\newcommand{\ULD}{\textsf{ULD}\xspace}
\newcommand{\LD}{\textsf{LD}\xspace}

\makeatletter
\newcommand*{\algotitle}[2]{%
  \NR@gettitle{#1}%
  \label{#2}%
}
\makeatother

\usepackage{natbib}

\usepackage{minitoc}

\title{Accelerated Mixing Time of Randomized Hamiltonian Monte Carlo}

\allowdisplaybreaks
\date{\today}
\author{Siddharth Mitra\thanks{Department of Computer Science, Yale University. Email: \texttt{siddharth.mitra@yale.edu}.} 
\and Vishwak Srinivasan\thanks{Department of Electrical Engineering and Computer Science, MIT. Email: \texttt{vishwaks@mit.edu}.} 
\and Xiuyuan Wang\thanks{Department of Computer Science, Yale University. Email: \texttt{xiuyuan.wang@yale.edu}.} 
\and Andre Wibisono\thanks{Department of Computer Science, Yale University. Email: \texttt{andre.wibisono@yale.edu}.} }

\begin{document}

\pagenumbering{roman} 
\maketitle
\vspace{2cm}

\begin{abstract}
    We show that idealized Randomized Hamiltonian Monte Carlo (\textsf{RHMC}) has accelerated mixing guarantees for log-concave densities $\nu\propto e^{-f}$ supported on $\mathbb{R}^d$, where $f$ is twice continuously differentiable with Lipschitz-continuous gradient.
    \textsf{RHMC} alternates exact Hamiltonian flow for independent random integration times with Gaussian velocity refreshments.
    Our guarantee holds for every integration-time  distribution with non-increasing log-concave density on $[0,\infty)$, including the uniform, triangular, and exponential distributions.
    We measure accuracy by the \textsf{KL} divergence of the output law after a prescribed number of iterations, and cost by the total expected integration time.
    For a log-concave target satisfying Talagrand's inequality with constant $\alpha$, we use a constant integration-time scale of order $1/\sqrt{\alpha}$ and obtain a total expected integration time of order $(1/\sqrt{\alpha})\log(1/\varepsilon)$ to reach error $\varepsilon$.
    Without assuming Talagrand's inequality, we use geometrically increasing integration-time scales and obtain a total expected integration time of order $1/\sqrt{\varepsilon}$, with a constant that depends on the initial \textsf{KL} divergence and Wasserstein distance to the target.
    The proof uses an averaging property of Hamiltonian flow under convexity and is inspired by an analogous acceleration result for optimization.
\end{abstract}    

\clearpage
\setcounter{tocdepth}{2}
\begingroup
\small
\tableofcontents
\endgroup
%%%%%%%%%%

\clearpage

%%%%%%%%%
\pagenumbering{arabic}
\section{Introduction}
\label{Sec:Introduction}

Drawing samples from a probability distribution is an essential algorithmic task in many areas of science and engineering.
A popular approach to sampling is to simulate a Markov chain whose stationary distribution is the target distribution; such methods are known as Markov Chain Monte Carlo (MCMC) \cite{brooks2011handbook}.
A central quantity in the analysis of MCMC is the \emph{mixing time}, which quantifies how quickly the Markov chain converges to its stationary distribution.
In practice, many Markov chains are time-discretizations of continuous-time dynamics, including
the overdamped or underdamped Langevin dynamics which are stochastic processes, and the Hamiltonian dynamics which is deterministic.
In this paper, we focus on the \emph{Randomized Hamiltonian Monte Carlo} (\RHMC) algorithm, which is based on simulating the deterministic Hamiltonian dynamics for integration times that are randomized, and show that it achieves accelerated mixing time guarantees.

To contextualize our results for \RHMC, we review relevant guarantees for the overdamped and underdamped Langevin dynamics, which have been studied more extensively.
We denote the target distribution that we aim to draw samples from by \(\nu^{X}\).
The overdamped Langevin dynamics (\textsf{LD}) \cite{roberts1996exponential} is a stochastic process for sampling \(\nu^{X}\) that uses the score function of \(\nu^{X}\) and is driven by a Brownian motion.
The Langevin dynamics is a natural process for sampling, and it has an optimization interpretation as the gradient flow dynamics (in the space of probability measures with the Wasserstein metric) to minimize the relative entropy or Kullback-Leibler (\(\KL\)) divergence with respect to the target distribution~\cite{JKO98}.
Partly motivated by this optimization perspective, a large body of work has developed a rich understanding of the mixing time of \textsf{LD} and its algorithmic implementations under various structural assumptions on the target distribution, with rates that match optimization results.
We highlight two key results that are relevant to our work:
\begin{itemize}
    \item When \(\nu^{X}\) satisfies an \(\alpha\)-\emph{log-Sobolev inequality},
    the mixing time of \textsf{LD} to reach error \(\varepsilon\) in \(\KL\) divergence scales as $(1 / \alpha) \log (1 / \varepsilon)$ \citep{OTTO2000361}.
    \item When \(\nu^{X}\) is \emph{log-concave}, the mixing time of \textsf{LD} to reach error \(\varepsilon\) in \(\KL\) divergence scales as $1/\varepsilon$ \citep{OV01}.
\end{itemize}
Notably, these rates match the time complexities of the gradient flow dynamics for minimizing a function to reach error \(\varepsilon\) when the objective function satisfies \(\alpha\)-gradient domination (which is implied by $\alpha$-strong convexity) and convexity, respectively.

Motivated by the algorithmic applications, as well as taking inspiration from the theory of acceleration in convex optimization, there have been efforts to identify other dynamics for sampling that mix faster than \LD.
Concretely, we seek the sampling analogues of the following results from convex optimization for the \emph{accelerated} gradient flow \cite{su2016differential,wilson2021lyapunov} to reach error \(\varepsilon\):
\begin{itemize}
    \item When the objective function is \(\alpha\)-strongly convex, the time complexity scales as \((1 / \sqrt{\alpha}) \log(1 / \varepsilon)\).
    \item When the objective function is convex, the time complexity scales as \(1 / \sqrt{\varepsilon}\).
\end{itemize}
One approach to accelerated sampling is to extend accelerated gradient flow from Euclidean space to the space of probability measures, as studied by~\citep{wang2022acceleratedinformationgradientflow,chen25accelerating}.
A different approach is to use ``kinetic'' dynamics that preserve the target distribution in the position marginal.

A candidate kinetic dynamics for accelerated sampling is the underdamped Langevin dynamics (\textsf{ULD}): a stochastic process defined on the phase space of both position and velocity.
In \textsf{ULD}, the evolution of the position is deterministically governed by the velocity, and the evolution of the velocity is encoded as a stochastic process involving the score function of \(\nu^{X}\).
\textsf{ULD} is natural to consider due to its structural similarities to the accelerated gradient flow in optimization.
Many works have been interested in understanding if \textsf{ULD} has faster convergence guarantees compared to \textsf{LD}; see for instance, \cite{cheng2018underdamped,ma2019analognesterovaccelerationmcmc}.
A notable prior approach for the analysis of \ULD is the space-time Poincar\'e approach of~\cite{Albritton_2024}.
In the table below, we summarize more recent mixing time guarantees that have been proven for \ULD.
\begin{table}[H]
\centering
\begingroup
\small
\renewcommand{\arraystretch}{1.35}
\setlength{\tabcolsep}{4pt}
\begin{tabular*}{\textwidth}{@{\extracolsep{\fill}}l l c c c@{}}
\textbf{Reference}
& \textbf{Setting for \(\nu^X\)}
& \textbf{Divergence}
& \textbf{Mixing Time}
& \textbf{Accelerated?} \\
\midrule
\cite{Cao_2023}
& Log-concave + \(\alpha\)-PI
& \(\chi^2\)
& \((1 / \sqrt{\alpha}) \log(1/ \varepsilon)\)
& Yes \\
\cite{altschuler2025shifted}
& \(\alpha\)-strongly log-concave
& \(\KL,\mathsf{R}_q\)
& \((\sqrt{L} /\alpha) \log(1 / \varepsilon)\)
& No \\
\cite{lu2026sharphypocoerciveentropydecay}
& Log-concave + \(\alpha\)-LSI
& \(\KL\)
& \((1 / \sqrt{\alpha})\log(1 / \varepsilon)\)
& Yes \\
\cite{li2026spacetime}
& Log-concave + \(\alpha\)-LSI
& \(\mathsf{R}_q\)
& \((1 / \sqrt{\alpha})\log(1 / \varepsilon)\)
& Yes \\
\cite{altschuler2025shifted}
& Log-concave
& \(\KL\)
& \(\sqrt{L} / \varepsilon\)
& No \\
\bottomrule
\end{tabular*}
\endgroup
\caption{
Summary of scalings of mixing time of \textsf{ULD} to output a distribution within \(\varepsilon\) to \(\nu^{X}\) in prior works. PI and LSI denote the Poincar\'e and log-Sobolev inequalities, respectively.
\(\KL\), \(\mathsf{R}_q\), and \(\chi^2\) denote the \textsf{KL}, \(q\)-R\'enyi, and chi-squared divergences, respectively. 
The bounds from \cite{altschuler2025shifted} additionally
assume \(L\)-log-smoothness.}
\label{tab:ULD-complexity}
\end{table}

In this work, we demonstrate the Hamiltonian dynamics as an alternative mechanism to obtain accelerated mixing time for sampling.
The Hamiltonian dynamics is a deterministic process on the phase space of position and velocity that conserves the Hamiltonian or energy function. 
The Hamiltonian Monte Carlo (\HMC) algorithm~\cite{duane1987hybrid} is a piecewise-deterministic Markov chain that proceeds iteratively by: (1) simulating the Hamiltonian dynamics for some integration time, and (2) resetting the velocity to be an independent Gaussian random variable.
\HMC and its variants are some of the most popular sampling algorithms in practice, used for instance in probabilistic programming languages such as Stan~\cite{carpenter2017stan} and PyMC~\cite{patil2010pymc}.

Despite its practical importance, the theoretical guarantees of \HMC and its variants are less developed than \textsf{LD} or \textsf{ULD}.
Many existing results focus on the \textit{short} integration time, where 
\HMC has unaccelerated convergence rates similar to \textsf{LD} and with additional dependence on the log-smoothness parameter.
\cite{mangoubi2021mixing} gave Wasserstein convergence bounds for ideal \HMC with a short, fixed integration time.
When \(\nu^{X}\) is \(\alpha\)-strongly log-concave and \(L\)-log-smooth, \cite[Lemma~2.1]{chen2019optimal} shows via a synchronous coupling approach that \HMC with integration time \(1/(2\sqrt{L})\) establishes a contraction in Wasserstein-2 distance.
Therefore, after \(O((L / \alpha) \log(1 / \varepsilon))\) iterations, we achieve a distribution that is \(\varepsilon\) away from \(\nu^{X}\) in Wasserstein-2 distance, and this amounts to total integration time that scales as \((\sqrt{L} / \alpha) \log(1 / \varepsilon)\).
They also show that for every fixed choice of the integration time, the relaxation time of ideal \HMC is \(\Omega(L/\alpha)\) in the worst case~\cite[Theorem~1.4]{chen2019optimal}.
Later work by \cite{monmarche2024entropic} extends this result by generalizing to target distributions that satisfy an $\alpha$-log-Sobolev inequality and are $L$-log-smooth, and providing a \(\KL\)-mixing time of the same order; 
\cite[Theorem~4]{monmarche2024entropic} also analyzes short random integration times. 

These results motivate looking beyond a short integration time to achieve faster convergence. 
Furthermore, a fixed integration time can resonate with the underlying Hamiltonian dynamics and yield no progress, even in the simpler setting when \(\nu^{X}\) is a Gaussian distribution.
In the setting where \(\nu^{X}\) is a Gaussian distribution, \cite{wang2022chebyshev} manage to derive accelerated
Wasserstein mixing time for ideal \HMC using a deterministic sequence
of integration times based on Chebyshev polynomials, thus highlighting the importance of choosing integration times carefully in this simpler setting.

While the original \HMC algorithm~\cite{duane1987hybrid} was formulated with a fixed integration time, \cite{Mackenzie1989HybridMonteCarlo} advocated for randomizing trajectory lengths to avoid such resonances and permit longer trajectories.
This motivates \textit{Randomized Hamiltonian Monte Carlo (\RHMC)},
which randomizes the integration times between velocity refreshments;
see also~\cite{cances2007theoretical,neal2011mcmc,bou2017randomized}.

Such randomization leads to accelerated sampling when \(\nu^{X}\) is Gaussian, as demonstrated in \cite{jiang2022dissipation} (along with partially refreshing the velocity).
\cite{apers2022gaussian} also analyze ideal \HMC for Gaussian targets with long integration times, and give gradient-query bounds for the discrete-time Metropolis-adjusted algorithm.
More generally, for a log-concave target that satisfies an $\alpha$-Poincar\'e inequality, \cite{Lu_2022} show that continuous-time \RHMC with Poisson velocity refreshment reaches error $\varepsilon$ in $\chi^2$ divergence at a time of order $(1/\sqrt{\alpha})\log(1/\varepsilon)$.
\cite{eberle2026nonreversible} obtain the same accelerated scaling in $\chi^2$ divergence using a non-reversible lifting framework.
For a log-concave target that satisfies an $\alpha$-log-Sobolev inequality, \cite{monmarche2026entropicconvergencepiecewisedeterministic} extend the technique of~\cite{lu2026sharphypocoerciveentropydecay,li2026spacetime} and obtain the same time scaling in \textsf{KL} divergence under additional regularity assumptions on the potential.

%%%%%%%%%
\subsection{Randomized Hamiltonian Monte Carlo}\label{subsec:RHMC}

\begin{algorithm}[h!t!]
    \caption{Randomized Hamiltonian Monte Carlo (\textsf{RHMC}) }
    \algotitle{\textsf{RHMC}}{alg:RHMC}
    \begin{algorithmic}
        \renewcommand{\REQUIRE}{\item[]\algorithmicrequire\ }
        \REQUIRE{Number of iterations $K\in\bbN$; base density $w$ on $[0,\infty)$; scale constants $T_1,\ldots,T_K>0$; initial distribution $\rho_0^X$.}
        \STATE{Draw initial iterate $x_0 \sim \rho_0^X$.}
        \FOR{$k=1,\dots,K$}
            \STATE{Set $X_0^{(k)}=x_{k-1}$, and draw $Y_0^{(k)}\sim \N(0, \I_d)$ independently.}
            \STATE{Draw $U_k\sim w$ independently, and set integration time $\tau_k=T_k \cdot U_k$.}
            \smallskip
            \STATE{Run Hamiltonian flow~\eqref{eq:HamFlow} from $(X_0^{(k)}, Y_0^{(k)})$ for time $\tau_k$ to obtain $(X_{\tau_k}^{(k)}, Y_{\tau_k}^{(k)})$.}
            \STATE{Set next iterate $x_{k} = X_{\tau_k}^{(k)}$ which is a random variable $x_{k} \sim \rho_{k}^X$.}
        \ENDFOR
        \RETURN{$x_K$ which is a random variable $x_K \sim \rho_K^X$.}
    \end{algorithmic}
\end{algorithm}

We introduce the Randomized Hamiltonian Monte Carlo (\RHMC) algorithm in more detail; see~\cref{alg:RHMC} for a formal description. 
Recall our goal is to sample from a target distribution $\nu^X$ with full support on $\R^d$.
We assume the target distribution has density function $\nu^X \propto e^{-f}$ for some potential function $f \colon \R^d \to \R$.
The \RHMC algorithm is based on the Hamiltonian dynamics for a \emph{Hamiltonian function} \((x, y) \mapsto H(x, y)\) which is defined on the phase space \(\R^{2d} = \R^{d} \times \R^{d}\) that consists of position \(x \in \R^{d}\) and velocity \(y \in \R^{d}\).
We define the Hamiltonian function to be
\begin{align}\label{Eq:HamDef}
    H(x,y) \deq f(x) + \frac{1}{2} \|y\|^2 \,.
\end{align}
Given the Hamiltonian function \eqref{Eq:HamDef}, the \textit{Hamiltonian dynamics} or \textit{Hamiltonian flow} is the evolution of the joint variables $(X_t, Y_t) \in \R^{2d}$ following the system of ordinary differential equations:
\begin{equation}
\label{eq:HamFlow}
\tag{\textsf{HF}}
\begin{aligned}
\dot{X}_{t} &= \nabla_y H(X_t,Y_t) = Y_t \\
\dot{Y}_{t} &= -\nabla_x H(X_t,Y_t) = -\nabla f(X_t)
\end{aligned}
\end{equation}
starting from any initial configuration $(X_0, Y_0) \in \R^{2d}$ at time $t = 0$. 

The \RHMC algorithm proceeds by evolving \((X_{t}, Y_{t})\) along the dynamics~\eqref{eq:HamFlow}, combined with velocity refreshment from the standard Gaussian distribution $\N(0, \I_d)$ at random times.
With a deterministic integration time, it is usually called \textit{Hamiltonian Monte Carlo} (\HMC).
We focus on the setting where the integration times are randomized~\cite{Mackenzie1989HybridMonteCarlo,cances2007theoretical,neal2011mcmc,bou2017randomized}.
The works~ \cite{Mackenzie1989HybridMonteCarlo,cances2007theoretical} did not specify the distribution of the integration times, but the subsequent works~\cite{bou2017randomized,monmarche2026entropicconvergencepiecewisedeterministic} set this to be an exponential distribution.
We study a general version of \RHMC where the random integration time is of the form \(T \cdot U\), where \(T > 0\) is a deterministic scale constant and \(U\) is sampled independently from a base density \(w\) over \([0, \infty)\) with \(\E_{w}[U^{2}] = 1\).

%%%%%%%%%%%%%
\subsection{Our contributions}
\label{Sec:Contribution}

We bound the \(\KL\) divergence of the output distribution of \RHMC after a prescribed number of iterations, and measure the cost by the expected sum of the integration times over the iterations.
We establish \emph{accelerated} mixing-time bounds for \RHMC with a log-concave, log-smooth target distribution \(\nu^{X}\) in two settings, summarized in~\Cref{tab:rhmc-result-scaling}.

\begin{table}[H]
\begin{center}
\renewcommand{\arraystretch}{1.25}
\begin{tabular}{c|c|c|c}
\textbf{Setting for \(\nu^{X}\)} & \makecell{\textbf{Base density} \\ \(w\)} & \makecell{\textbf{Scale constants} \\ \(\{T_{k}\}_{k \geq 1}\)} & \textbf{\makecell{Expected total integration \\ time for \(\KL\) error \(\varepsilon\)}} \\
\hline
\makecell{log-concave + \\
\(\alpha\)-Talagrand's inequality} & \multirow{3}{*}{\makecell{log-concave, \\
non-increasing, \\ with \\
\(\E_{U \sim w}[U^{2}] = 1\)}} & \makecell{constant \\ \(2/\sqrt{\alpha}\)} & \makecell{\((1/\sqrt{\alpha}) \log (1/\varepsilon) \) \\ {\scriptsize (\cref{thm:general-ti,cor:ConvRHMC_SLC})}} \\[1em] 
\arrayrulecolor{lightgray}
\cline{1-1}
\cline{3-4}
log-concave & & \makecell{exponentially \\
increasing} & \makecell{\(1/\sqrt{\varepsilon}\) \\ {\scriptsize (\cref{thm:ConvRHMC_LC,cor:ConvRHMC_LC})}} \\
\hline
\end{tabular}
\end{center}
\caption{Summary of integration-time bounds for \(\RHMC\). Initialization factors are omitted.}
\label{tab:rhmc-result-scaling}
\end{table}

In the first setting of log-concave target satisfying Talagrand's inequality, we obtain a \(1/\sqrt{\alpha}\) scaling in the \(\KL\)-mixing time, which improves the \(1 / \alpha\) scaling of the overdamped Langevin dynamics under log-Sobolev inequality.
For exponential integration times, our bound has the same $1/\sqrt{\alpha}$ time scaling as that of Monmarch\'e and Wang~\cite{monmarche2026entropicconvergencepiecewisedeterministic} for log-concave targets satisfying both sets of assumptions; 
they bound the joint \textsf{KL} divergence at deterministic times, whereas we bound the position \textsf{KL} divergence after a prescribed number of iterations.
Our result in this first setting also applies to mildly non-log-concave targets; see~\cref{thm:general-ti} for details.
In the second setting of log-concave target (without assuming Talagrand's inequality), we obtain a $1/\sqrt{\varepsilon}$ scaling in the \(\KL\)-mixing time, which improves the $1 / \varepsilon$ scaling of the overdamped Langevin dynamics in the same setting.
Our result in this second setting is new, to the best of our knowledge.
Our results in both settings match the improvement that one expects from the theory of accelerated convex optimization.
Our results are not limited to the case of exponential integration time, as studied in prior work~\cite{bou2017randomized,monmarche2026entropicconvergencepiecewisedeterministic}.
Furthermore, unlike the short, fixed integration-time bounds discussed above, our bounds do not depend on the log-smoothness parameter.
Our bounds on the total integration time hold in expectation.
When the integration time distribution has bounded support, we also give an almost-sure bound; see~\cref{cor:ConvRHMC_SLC,cor:ConvRHMC_LC} for details.

Our analysis for \textsf{RHMC} is inspired by the connections between optimization and sampling, and in particular the ideas in our previous work \cite{wang26} where we studied an accelerated optimization algorithm based on the Hamiltonian dynamics;
see~\Cref{app:HMCandHFopt} for further discussion.
We note that our analysis deviates from the techniques in \cite{bou2017randomized,monmarche2026entropicconvergencepiecewisedeterministic}, who adopt a generator calculus approach for analyzing \RHMC with integration time distribution being an exponential distribution.
Our analysis handles integration-time distributions with a non-increasing, log-concave density (of which an exponential distribution is an example) by separating the analysis of the Hamiltonian dynamics from that of the integration time distribution; see~\cref{sec:evolution_along_HF,sec:MainProofs} for details.

%%%%%%%%%%
\section{Background}
\label{Sec:Background}

Let $\P(\R^d)$ denote the space of all probability distributions on $\R^d$. 
We denote by $\P_2(\R^d)$ the subset of \(\P(\R^{d})\) with finite second moment, and we denote by \(\P_{\ac}(\R^{d})\) the subset of \(\P(\R^{d})\) that are absolutely continuous with respect to the Lebesgue measure. 
We define the intersection $\mathcal P_{2,\mathrm{ac}}(\mathbb R^d) := \mathcal P_2(\mathbb R^d)\cap\mathcal P_{\mathrm{ac}}(\mathbb R^d)$. 
We denote by $\PtwoacfsRd$ the subset of $\mathcal P_{2,\mathrm{ac}}(\mathbb R^d)$ that admit a strictly positive density on $\mathbb R^d$.
Every distribution in $\PtwoacfsRd$ has full support. 
For \(\rho \in \PtwoacfsRd\), we also denote its probability density function by \(\rho \colon \R^d \to \R\).
For $\rho, \nu \in \P(\R^d)$, we write $\rho \ll \nu$ to denote $\rho$ is absolutely continuous with respect to $\nu$, i.e., if $\nu(A)=0$ for any measurable $A \subseteq \R^d$, then $\rho(A) = 0$. 
For $\rho, \nu \in \P(\R^d)$, let $\Pi(\rho, \nu)$ denote the set of couplings of $\rho$ and $\nu$, i.e., joint probability distributions on $\R^{2d}$ with marginal distributions $\rho$ and $\nu$. 
For a measurable map $S$ and a probability distribution $\rho$, we write $S_\#\rho$ for the law of $S(X)$ when $X\sim\rho$.
For matrices of the same size, let
$\langle A,B\rangle_{\mathsf F}=\Tr(A^\top B)$ denote the Frobenius
inner product, and let $\|A\|_{\mathsf F}=\sqrt{\langle A,A\rangle_{\mathsf F}}$
denote the corresponding norm.
We denote by $\gamma_d$ the standard Gaussian distribution \(\N(0, \I_{d})\) on \(\R^{d}\).

%%%%%%%%%%%%%
\paragraph{Wasserstein distance and \textsf{KL} divergence.} 
For $\rho, \nu \in \P_2(\R^d)$, the \textit{Wasserstein--$2$ distance} between $\rho$ and $\nu$ is 
\begin{equation*}
    \Wass_2(\rho, \nu) \coloneqq \left( \inf_{\omega \in \Pi(\rho, \nu)} \E_{(x,y) \sim \omega} \Big[ \|x-y\|^2 \Big] \right)^{\frac{1}{2}} \,.
\end{equation*}
This is a metric over \(\P_{2}(\R^{d})\), and induces a rich geometry over the space of probability measures.
For \(\rho, \nu \in \P_{\ac}(\R^{d})\) with \(\rho \ll \nu\), the \textit{Kullback-Leibler} (\(\KL\)) \textit{divergence} or \textit{relative entropy} of \(\rho\) with respect to \(\nu\) is 
\begin{equation*}
    \KL(\rho \dvert \nu) \deq \E_{\rho}\left[\log\frac{\rho}{\nu}\right] = \int_{\R^d}\rho(x)\log\frac{\rho(x)}{\nu(x)}\dx \,.
\end{equation*}
The $\KL$ divergence is not a metric over \(\P_{\ac}(\R^{d})\); in particular, it is not symmetric and does not satisfy the triangle inequality.
The Wasserstein--$2$ distance and the \textsf{KL} divergence are nonnegative and vanish if and only if the two distributions are equal.

%%%%%%%%%%%%%
\paragraph{Convexity and smoothness.}
Let $f \colon \R^{d} \to \R$ be a differentiable function.
Recall the Bregman divergence $D_f \colon \R^d \times \R^d \to \R$ is 
$$D_f(x,y) = f(x) - f(y) - \langle \nabla f(y), x-y \rangle $$
for all \(x, y \in \R^{d}\).
We say that \(f \colon \R^d \to \R\) is \textit{convex} if \(D_{f}(x, y) \geq 0\) for all \(x, y \in \R^{d}\).
We say $f$ is \emph{\(\alpha\)-strongly convex} for some $\alpha > 0$ if the function \(x \mapsto f(x) - \frac{\alpha}{2}\|x\|^{2}\) is convex, or equivalently, $D_f(x,y) \ge \frac{\alpha}{2} \|x-y\|^2$ for all $x,y \in \R^d$.
We say $f$ is \emph{\(M\)-semi-convex} for some $M \ge 0$ if the function \(x \mapsto f(x) + \frac{M}{2}\|x\|^{2}\) is convex, or equivalently, $D_f(x,y) \ge -\frac{M}{2} \|x-y\|^2$ for all $x,y \in \R^d$.
We note $\alpha$-strong convexity implies convexity (the case $\alpha = 0$), and convexity implies $M$-semi-convexity (the case $M = 0$).
We say \(f\) is \textit{\(L\)-smooth} for some $L \in (0,\infty)$ if 
$| D_f(x,y) | \le \frac{L}{2} \|x-y\|^2$ for all $x,y \in \R^d$. 
We say that $f$ is \textit{smooth} if $f$ is $L$-smooth for some $L \in (0,\infty)$.

Let \(\nu \in \PtwoacfsRd\) with a differentiable density function.
We say that \(\nu\) is \textit{log-concave}, \textit{\(\alpha\)-strongly-log-concave}, and \textit{\(M\)-semi-log-concave} if \(-\log \nu \colon \R^d \to \R\) is convex, \(\alpha\)-strongly convex, and \(M\)-semi-convex, respectively. 
We say that $\nu$ is \textit{$L$-log-smooth} if $\log \nu$ is $L$-smooth.
We say $\nu$ is \textit{log-smooth} if it is $L$-log-smooth for some $L\in (0, \infty)$.

%%%%%%%%%%%
\paragraph{Functional inequalities.}
We say $\nu \in \P_{2, \ac, \fs}(\R^d)$ satisfies \textit{Talagrand's inequality} with constant $\alpha > 0$ if for any $\rho \in \P_{2, \ac, \fs}(\R^d)$, we have
\[
\frac{\alpha}{2}\,\sfW_2^2(\rho, \nu) \leq \KL(\rho \dvert \nu)\,.
\]
We say $\nu \in \P_{2, \ac, \fs}(\R^d)$ satisfies a \textit{log-Sobolev inequality} with constant $\alpha > 0$ if for every $\rho \in \P_{2, \ac, \fs}(\R^d)$,
\[
    \KL(\rho \dvert \nu) \leq \frac{1}{2\alpha}\FI(\rho \dvert \nu) \,.
\]
Here $\FI(\rho \dvert \nu)=\E_{\rho}[\|\nabla \log(\rho/\nu)\|^2]$ is the relative Fisher information. We set $\FI(\rho \dvert \nu)=+\infty$ if the density ratio $\rho/\nu$ is not strictly positive and differentiable.

We recall Talagrand's inequality and log-Sobolev inequality can be viewed as the quadratic growth and gradient domination conditions for the objective function \(\mathcal{F}(\rho) = \KL(\rho \dvert \nu)\), respectively, in the space of probability measures with the \(\sfW_2\)  metric; see~\cite{OTTO2000361,wibisono2018sampling, VW23} for further discussion.
In particular, we recall the following relationships from~\cite{OTTO2000361}:
\begin{itemize} 
    \item If $\nu$ is \(\alpha\)-strongly log-concave, then $\nu$ satisfies \(\alpha\)-log-Sobolev inequality.
    \item If \(\nu\) satisfies \(\alpha\)-log-Sobolev inequality, then \(\nu\) satisfies \(\alpha\)-Talagrand's inequality.
    \item If \(\nu\) satisfies \(\alpha\)-Talagrand's inequality and is \(M\)-semi-log-concave for some \(\alpha > M \geq 0\), then \(\nu\) satisfies \(\beta\)-log-Sobolev inequality with \(\beta = \frac{\alpha}{4} \left(1-\frac{M}{\alpha}\right)^2\).
\end{itemize}

Our first main result in~\Cref{thm:general-ti} holds when the target distribution $\nu^X$ satisfies \(\alpha\)-Talagrand's inequality and is $M$-semi-log-concave for some $\alpha > M \ge 0$, and therefore, \(\nu^{X}\) also satisfies a log-Sobolev inequality.
Recent accelerated rates for \ULD \cite{lu2026sharphypocoerciveentropydecay,li2026spacetime} and \RHMC \cite{monmarche2026entropicconvergencepiecewisedeterministic} have also assumed the combination of semi-log-concavity and Talagrand's inequality or a log-Sobolev inequality for the target distribution.

\paragraph{Standing assumption.}
Throughout, we make the following assumption.

\begin{assumption}\label{assump:target-regularity}
The target distribution $\nu^X \in \PtwoacfsRd$ is log-smooth and $-\log \nu^X \colon \R^d \to \R$ is twice continuously differentiable.
\end{assumption}

This assumption ensures the necessary regularity conditions are satisfied for the Hamiltonian flow (\ref{eq:HamFlow}); see \cref{apdx:HF_review} for further discussion.
We note the log-smoothness parameter of $\nu^X$ does not quantitatively affect the mixing time guarantees we derive.

%%%%%%%%%%%%%%%%
\section{Main mixing time results for \RHMC}\label{subsec:MainResults}

We present our results on the \(\KL\)-mixing time guarantees for \RHMC for sampling from a log-smooth distribution \(\nu^{X} \in \PtwoacfsRd\) in the settings where \(\nu^{X}\) satisfies either a mild semi-log-concavity and Talagrand's inequality, or only log-concavity.

Throughout, we use the following assumption on the base density \(w\) for the integration time.
\begin{assumption}
\label{assump:integration-time}
The function \(w\) is a non-increasing, log-concave density function over \([0, \infty)\) that satisfies \(\E_{U \sim w}[U^{2}] = 1\), and sampling from \(w\) is computationally easy.
\end{assumption}

The normalization \(\E_{U \sim w}[U^{2}] = 1\) fixes the scale of the integration time.
Some examples of $w$ that satisfy \cref{assump:integration-time} are:
\begin{itemize}[itemsep=-4pt]
    \item the uniform distribution on \([0,\sqrt{3}]\) with \(w(t) = \frac{1}{\sqrt{3}} \),
    \item the triangular distribution supported on \([0,\sqrt{6}]\) with \(w(t)= \frac{2}{\sqrt{6}} \left(1-\frac{t}{\sqrt{6}} \right) \), and  
    \item the exponential distribution supported on \([0,\infty)\) with \(w(t)=\sqrt{2}e^{-\sqrt{2}t}\).
\end{itemize}
The uniform and triangular distributions have bounded support, so they give almost-sure bounds on the total integration time. 
The exponential distribution has unbounded support, so our bounds on the total integration time hold in expectation.

%%%%%%%%%%%%%%%%
\subsection{When \texorpdfstring{\(\nu^{X}\)}{the target distribution} satisfies semi-log-concavity and Talagrand's inequality}
\label{subsec:MainResultsTI}

We state our first main result when the target distribution $\nu^X$ satisfies a mild semi-log-concavity and Talagrand's inequality.
As noted earlier, this also implies $\nu^X$ satisfies log-Sobolev inequality.

\begin{theorem}
\label{thm:general-ti}
Let $\nu^X$ satisfy~\cref{assump:target-regularity}.
Assume \(\nu^{X}\) is \(M\)-semi-log-concave and satisfies $\alpha$-Talagrand's inequality for some \(0 \leq M < \alpha\).
Suppose the initial distribution \(\rho_{0}^{X} \in \PtwoacfsRd\) satisfies \(\KL(\rho_{0}^{X} \dvert \nu^{X}) < \infty\), and the base density \(w\) satisfies \cref{assump:integration-time}.
If we set \(T_{k} = \frac{2}{\sqrt{\alpha - M}}\) for all \(k \geq 1\), then the output \(x_{K} \sim \rho_{K}^{X}\) of \RHMC satisfies for every integer $K \geq 1$:
\begin{equation}\label{eq:general-ti-convergence}
\KL(\rho_{K}^X\dvert\nu^X)\leq \left(1-\frac{\alpha-M}{18\alpha}\right)^{K} \cdot \KL(\rho_0^X\dvert\nu^X) \,.
\end{equation}
\end{theorem}
We provide the proof of~\Cref{thm:general-ti} in~\Cref{sec:proof-thm-ti}.
This leads to the following corollary about the \(\KL\)-mixing time of \RHMC.
We provide the proof of~\Cref{cor:ConvRHMC_SLC} in~\Cref{app:proof-cor-ti}.

\begin{corollary}
\label{cor:ConvRHMC_SLC}
Assume the same assumptions as in~\Cref{thm:general-ti}.
For any $0 < \varepsilon < \KL(\rho_{0}^{X} \dvert \nu^{X})$,
to reach error $\KL(\rho_{K}^{X} \dvert \nu^{X}) \le \varepsilon$, it suffices to run \RHMC with the setting specified in~\Cref{thm:general-ti} for $K$ iterations where
\begin{equation*}
        K =  \left\lceil\frac{1}{\log\left(\frac{18\alpha}{17\alpha+M}\right)}\, \log\frac{\KL(\rho_0^X\dvert\nu^X)}{\varepsilon}\right\rceil \,.
\end{equation*}
The expected total integration time satisfies
\[
    \E\left[\sum_{k=1}^K\tau_k\right] \leq 38 \cdot \frac{\alpha}{(\alpha-M)^{3/2}}\left\lceil\log\frac{\KL(\rho_0^X\dvert\nu^X)}{\varepsilon}\right\rceil \,.
\]
If $w$ is supported on $[0,B]$ for some $B<\infty$, then we also have the following bound almost surely:
\[
    \sum_{k=1}^{K} \tau_{k} \leq 38 \cdot B \cdot \frac{\alpha}{(\alpha-M)^{3/2}}\left\lceil\log\frac{\KL(\rho_0^X\dvert\nu^X)}{\varepsilon}\right\rceil \,.
\]
\end{corollary}

%%%%%%%%%%%%%%%%
\subsection{When \texorpdfstring{\(\nu^{X}\)}{the target distribution} only satisfies log-concavity}
\label{subsec:MainResultLC}

We now state our second result when the target distribution only satisfies log-concavity.

\begin{theorem}
\label{thm:ConvRHMC_LC}
Let $\nu^X$ be log-concave and satisfy~\cref{assump:target-regularity}.
Suppose the initial distribution \(\rho_{0}^{X} \in \PtwoacfsRd\) satisfies \(\KL(\rho_{0}^{X} \dvert \nu^{X}) < \infty\), and the base density \(w\) satisfies \cref{assump:integration-time}.
If we set \(T_k=(9/8)^{k/2}\) for all \(k\geq 0\), then the output \(x_K \sim \rho_K^X\) of \RHMC satisfies for every integer $K \geq 1$:
\begin{equation}\label{eq:CleanConvergenceWLC}
\KL(\rho_K^X\dvert\nu^X)+\frac{1}{9}\left(\frac{8}{9}\right)^K\Wass_2^2(\rho_K^X,\nu^X)
\leq \left(\frac{8}{9}\right)^K\left(\KL(\rho_0^X\dvert\nu^X)+\frac{1}{9}\Wass_2^2(\rho_0^X,\nu^X)\right) \,.
\end{equation}
\end{theorem}

We provide the proof of~\Cref{thm:ConvRHMC_LC} in~\Cref{sec:proof-thm-lc}.
This leads to the following corollary about the \(\KL\)-mixing time of \RHMC.
We provide the proof of~\Cref{cor:ConvRHMC_LC} in~\Cref{app:proof-cor-lc}. 

\begin{corollary}\label{cor:ConvRHMC_LC}
Assume the same assumptions as in~\Cref{thm:ConvRHMC_LC}.
For any \(0<\varepsilon<\KL(\rho_0^X\dvert\nu^X)\), to reach \(\KL(\rho_K^X\dvert\nu^X)\leq\varepsilon\), it suffices to run \RHMC with the setting specified in~\Cref{thm:ConvRHMC_LC} for $K$ iterations where
\[ 
K=\left\lceil \frac{1}{\log(9/8)} \, \log\left(\frac{\KL(\rho_0^X\dvert\nu^X)+\frac{1}{9}\Wass_2^2(\rho_0^X,\nu^X) }{\varepsilon}\right) \right\rceil \,.
\]
The expected total integration time satisfies
\[
    \E\left[\sum_{k=1}^K\tau_k\right] \leq 19 \cdot \sqrt{\frac{\KL(\rho_0^X\dvert\nu^X)+\frac{1}{9}\Wass_2^2(\rho_0^X,\nu^X)}{\varepsilon}} \,.
\]
If $w$ is supported on $[0,B]$ for some $B<\infty$, then we also have the following bound almost surely:
\[
    \sum_{k=1}^K \tau_k \leq 19 \cdot B \cdot \sqrt{\frac{\KL(\rho_0^X\dvert\nu^X)+\frac{1}{9}\Wass_2^2(\rho_0^X,\nu^X)}{\varepsilon}} \,.
\]
\end{corollary}

We note that for both~\Cref{cor:ConvRHMC_SLC} and~\Cref{cor:ConvRHMC_LC}, the total integration time \(\sum_{k=1}^K\tau_k\) is a sum of independent nonnegative log-concave random variables, and is therefore also log-concave.
Thus, the expected value guarantees also imply high-probability guarantees by standard arguments, which we omit for simplicity.

%%%%%%%%%%%%%%%%%
\subsection{Proof technique: Averaging property along Hamiltonian dynamics}
\label{sec:evolution_along_HF}

The proofs of the main results rely on a bound on the average \textsf{KL} divergence of the $X$-marginal to the target distribution \(\nu^{X}\) along the trajectory of the deterministic Hamiltonian flow when $\nu^X$ is semi-log-concave; we present this bound in \cref{Lem:Wass-KL-integratedKey} in \cref{subsec:KeyLemma}. 
We first review general properties of the \textsf{KL} divergence along the Hamiltonian dynamics~\eqref{eq:HamFlow} in~\cref{subsec:KLPropertiesAndHamFlow}. 

%%%%%%%%%%%
\subsubsection{Distributional properties of the Hamiltonian dynamics}\label{subsec:KLPropertiesAndHamFlow}

Recall from~\cref{Eq:HamDef} that the Hamiltonian function $H \colon \R^{2d} \to \R$ is defined as $H(x,y) \deq f(x) + \frac{1}{2} \|y\|^2 $.
This induces the joint probability distribution $\nu^{XY} \propto e^{-H}$ on $\R^{2d}$.
Notice that $\nu^{XY} = \nu^X \otimes \gamma_{d}$.

Since we assume $f$ is $L$-smooth for some $L \in (0, \infty)$ (equivalent to \(\nabla f\) being globally \(L\)-Lipschitz), we have by the Picard-Lindel\"of theorem that from any initial condition $(X_0, Y_0) \in \R^{2d}$ the Hamiltonian flow~\eqref{eq:HamFlow} is well-defined and admits a unique solution $(X_t, Y_t) \in \R^{2d}$ for all $t \in \R$. 
Suppose we run the deterministic Hamiltonian flow~\eqref{eq:HamFlow} from a random variable $(X_0, Y_0) \sim \rho_0^{XY}$ drawn from some initial distribution $\rho_0^{XY}$.
Then at each time $t \in \R$, we obtain another random variable $(X_t, Y_t) \sim \rho_t^{XY}$.
The deterministic Hamiltonian flow conserves the $\KL$ divergence to $\nu^{XY}$, as shown in \cref{Lem:HamFlowConsvJointKL} (see~\cref{Sec:HamFlowConsvJointKLProof} for the proof).
This follows from two key properties of Hamiltonian flow: the conservation of the Hamiltonian value, and the conservation of Lebesgue volume on the phase space \(\R^{2d}\) (see \cref{subapp:HF-deterministic-properties} for a review).

\begin{lemma}\label{Lem:HamFlowConsvJointKL}
    Let $\rho_0^{XY} \in \P_{2,\ac,\fs}(\R^{2d})$ with $\KL(\rho_0^{XY} \dvert \nu^{XY}) < \infty$.
    For $t \in \R$, let $(X_t, Y_t) \sim \rho_t^{XY}$ be the solution to the Hamiltonian flow~\eqref{eq:HamFlow} from $(X_0, Y_0) \sim \rho_0^{XY}$.
    Then we have
    \begin{equation*}
        \KL(\rho_t^{XY} \dvert \nu^{XY}) = \KL(\rho_0^{XY} \dvert \nu^{XY}) \,.
    \end{equation*}
\end{lemma}

From~\Cref{Lem:HamFlowConsvJointKL}, we see that $\nu^{XY}$ is an invariant distribution under the Hamiltonian flow~\eqref{eq:HamFlow} i.e., if \(\rho_{0}^{XY} = \nu^{XY}\), then \(\rho_{t}^{XY} = \nu^{XY}\) for all \(t \in \R\).
However, $\nu^{XY}$ is not a unique invariant distribution.
In fact, for any function $\phi \colon \R \to \R$ such that $\int_{\R^{2d}} e^{-\phi(H(x,y))} \dx \dy < \infty$, the probability distribution $\tilde \nu^{XY} \propto \exp(-\phi(H))$ is also invariant under the Hamiltonian flow~\eqref{eq:HamFlow}.

What distinguishes $\nu^{XY} = \nu^X \otimes \gamma_{d} \propto e^{-H}$ in \RHMC is the Gaussian velocity refreshment. 
Among all distributions of the form $\tilde \nu^{XY} \propto  \exp(-\phi(H))$, $\nu^{XY}$ is the only distribution whose conditional distribution of $Y$ given $X$ is $\gamma_{d} = \N(0, \I_d)$. 
Thus, it is the only one that is invariant by the velocity refreshment, making $\nu^X$ the stationary distribution of \RHMC.
Moreover, if $\rho^{XY} = \rho^X \otimes \gamma_{d}$, then we see that $\KL(\rho^{XY} \dvert \nu^{XY}) = \KL(\rho^X \dvert \nu^X)$.
Thus, combining~\cref{Lem:HamFlowConsvJointKL} with the chain rule of KL divergence yields the following descent property for the $X$ marginal $X_t \sim \rho_t^X$. 
We provide the proof of~\Cref{Lem:DescentKL} in~\Cref{Sec:DescentKLProof}.

\begin{lemma}\label{Lem:DescentKL}
    Let $\rho_0^{XY} = \rho_0^X \otimes \gamma_{d}$ for some $\rho_0^X \in \P_{2,\ac,\fs}(\R^{d})$ with $\KL(\rho_0^{X} \dvert \nu^{X}) < \infty$.
    For $t \in \R$, let $(X_t, Y_t) \sim \rho_t^{XY}$ be the solution to the Hamiltonian flow~\eqref{eq:HamFlow} from $(X_0, Y_0) \sim \rho_0^{XY}$, and let $X_t \sim \rho_t^X$ denote the $X$-marginal of \(\rho_{t}^{XY}\). Then,
    \begin{equation*}
        \KL(\rho_{t}^{X} \dvert \nu^{X}) = \KL(\rho_{0}^{X} \dvert \nu^{X}) - \E_{x\sim\rho_t^X}\bigl[\KL(\rho_t^{Y\mid X=x}\dvert\gamma_{d})\bigr] \leq \KL(\rho_{0}^{X} \dvert \nu^{X}) \,.
    \end{equation*}
\end{lemma}

In the appendix, we review other regularity properties of \(\rho_{t}^{X}\) along the Hamiltonian flow: in~\Cref{Lem:KLContinuous} (\Cref{Sec:KLContinuousProof}) we show the map $t \mapsto \KL(\rho_t^X \dvert \nu^X)$ is continuous, and in~\Cref{lem:second-moment-finiteness} (\Cref{Sec:second-moment-finitenessProof}) we show that the second moment remains finite.

%%%%%%%%%%%%%%%%%%
\subsubsection{Bounding the average \textsf{KL} divergence}\label{subsec:KeyLemma}

We now present the key lemma bounding the average of the \textsf{KL} divergence of the $X$-marginal along the trajectory of the Hamiltonian flow when the target distribution $\nu^X$ is semi-log-concave.

\begin{lemma}\label{Lem:Wass-KL-integratedKey}
    Assume $\nu^X$ satisfies~\cref{assump:target-regularity} and is $M$-semi-log-concave for some $M\geq 0$.
    Suppose $\rho_0^X \in \P_{2,\ac,\fs}(\R^d)$ satisfies $\KL(\rho_0^X \dvert \nu^X)<\infty$, and set $\rho_0^{XY} = \rho_0^X \otimes \gamma_{d}$.
    For $t \ge 0$, let $(X_t, Y_t) \sim \rho_t^{XY}$ be the solution to the Hamiltonian flow~\eqref{eq:HamFlow} from $(X_0, Y_0) \sim \rho_0^{XY}$, and let $X_t \sim \rho_t^X$ denote the $X$-marginal of \(\rho_{t}^{XY}\).
    Then for all $0 \le T < \infty$, the following holds:
    \begin{multline}\label{Eq:KLIntegral}
        \frac{1}{2}\Wass_2^2(\rho_T^X, \nu^X) + 3 \int_0^T (T-t) \, \KL(\rho_t^X \dvert \nu^X ) \dt - \frac{M}{2} \int_0^T (T-t) \, \Wass_2^2(\rho_t^X, \nu^X) \dt \\
        \leq \frac{1}{2} \Wass_2^2(\rho_0^X, \nu^X) + T^2 \, \KL(\rho_0^X \dvert \nu^X) \,.
    \end{multline}
\end{lemma}

We provide a proof sketch below.
We provide the full proof of~\cref{Lem:Wass-KL-integratedKey} in~\Cref{Sec:Wass-KL-integratedKeyProof}.

\begin{proof}[Proof sketch]
We first give a formal derivation in Steps~1-4, where we treat the density, the conditional moments, and the maps introduced below as smooth enough to perform the differentiations and integrations by parts, and we treat $t\mapsto\Wass_2^2(\rho_t^X,\nu^X)$ as twice differentiable.
In Step~5 we explain how the rigorous proof proceeds without assuming a smooth optimal transport map or an ordinary second time derivative.

\paragraph{Step 1: Evolution of the conditional moments.}
Define the conditional mean, second moment, and covariance matrix of the velocity by
\begin{equation}
\label{eq:def-cond-moments}
u_t(x)\deq\E[Y_t\mid X_t=x] \,,\qquad
M_t(x)\deq\E[Y_tY_t^\top\mid X_t=x] \,,\qquad
\Sigma_t(x)\deq M_t(x)-u_t(x)u_t(x)^\top \,.
\end{equation}
Differentiating the expectations of suitable test functions along~\eqref{eq:HamFlow}
gives the continuity and momentum equations
(see~\Cref{lem:continuity-momentum-equation}):
\begin{align*}
\partial_t\rho_t^X &= -\nabla\cdot(\rho_t^Xu_t) \,,\\
\partial_t(\rho_t^Xu_t)& =-\nabla\cdot(\rho_t^XM_t)-\rho_t^X\nabla f \,.
\end{align*}
Substituting \(M_t=u_tu_t^\top+\Sigma_t\) into the momentum equation
and using the continuity equation to cancel the terms involving
\(\partial_t\rho_t^X\) gives the acceleration vector field (see~\Cref{lem:at_representation}):
\begin{equation}
    a_t(x) \,\deq\, \partial_tu_t(x)+\nabla u_t(x)u_t(x)
    \,=\, -\nabla f(x)-\frac{1}{\rho_t^X(x)}
    \nabla\cdot(\rho_t^X\Sigma_t)(x)\,.
    \label{eq:acceleration_formula}
\end{equation}

\paragraph{Step 2: Time derivative of the Wasserstein distance.}
We bound the second time derivative of $\Wass_2^2(\rho_t^X, \nu^X)$ in terms of $u_t$ and $a_t$ defined in Step~1; see~\cref{lem:Wass22-FD} for the rigorous derivation using the regularity established in~\cref{lem:secondary_regularity}.
Define the map $\Phi_t \colon \R^d \to \R^d$ by $\Phi_0(x)=x$ and $\dot\Phi_t(x)=u_t(\Phi_t(x))$, so that
\[
\ddot\Phi_t(x)=a_t(\Phi_t(x)) \,.
\]
As $(\Phi_t)_\#\rho_0^X$ and $\rho_t^X$ solve the same continuity equation (defined in Step~1), we have $(\Phi_t)_\#\rho_0^X = \rho_t^X$.
Fix \(t\) and let \(R_t\) be the optimal transport map from \(\rho_t^X\) to \(\nu^X\).
Then the joint distribution \((\Phi_{t\pm h},R_t\circ\Phi_t)_\#\rho_0^X\) is a coupling between \(\rho_{t\pm h}^X\) and \(\nu^X\). Thus, we can bound $\Wass_2^2(\rho_{t\pm h}^X, \nu^X)$ by
\[
\Wass_2^2(\rho_{t\pm h}^X, \nu^X) \leq \int_{\R^d} \|\Phi_{t\pm h}(x) - R_t(\Phi_t(x))\|^2  \rho_0^X(x) \dx \,,
\]
with equality at $h=0$. 
Comparing the second $h$-derivatives at $h=0$ yields, with \(v_t(x)\deq x-R_t(x)\):
\begin{align}
\frac{1}{2}\frac{\d^2}{\d t^2}\Wass_2^2(\rho_{t}^X, \nu^X)&\leq\int_{\R^d}
\bigl(\|u_t(x)\|^2+\langle v_t(x),a_t(x)\rangle\bigr)
\rho_t^X(x)\dx \, . \label{eq:second-derivative-w2}
\end{align}

\paragraph{Step 3: Relating the second derivative to \textsf{KL} divergence.}
We substitute the acceleration formula~\eqref{eq:acceleration_formula} into the right-hand side of~\eqref{eq:second-derivative-w2}.
We use three key steps: (a) integrating the covariance term by parts introduces the Jacobian $\nabla R_t(x)$ of the optimal transport map, (b)
comparing the conditional distribution of \(Y_{t}\) given \(X_{t} = x\) to both \(\mathcal{N}(u_{t}(x), (\nabla R_{t}(x))^{-1})\) and \(\gamma_{d}\), and (c) the Monge-Amp\'{e}re identity that relates the densities of \(\rho_{t}^{X}\) and \(\nu^{X}\).
At this stage, we use the \(M\)-semi-convexity of \(f\) to bound $- \langle v_t, \nabla f \rangle $ up to an error of \(\|v_{t}\|^{2}\).
In the end, we obtain (see~\Cref{Lem:KLBoundFormula}):
\begin{align}
    \frac{1}{2}\frac{\d^2}{\dt^2}\Wass_2^2(\rho_{t}^X, \nu^X)
&\leq 2\int_{\R^d}
\KL\!\left(\rho_t^{Y\mid X=x}\dvert\gamma_{d}\right)\rho_t^X(x)\dx
-\KL(\rho_t^X\dvert\nu^X)+\frac M2\Wass_2^2(\rho_t^X,\nu^X)\nonumber\\
&=2\KL(\rho_0^X\dvert\nu^X)-3\KL(\rho_t^X\dvert\nu^X)
+\frac M2\Wass_2^2(\rho_t^X,\nu^X) \,,\label{Eq:DiffIneqW2}
\end{align}
where the last equality follows from \cref{Lem:DescentKL}.

\paragraph{Step 4: Integrating in time.}
Integrating~\cref{Eq:DiffIneqW2} twice and using the fact 
$\frac{\d}{\dt}\Wass_2^2(\rho_t^X,\nu^X)=0$ at $t=0$ (since $u_0=0$), we obtain
the following, which rearranges to~\cref{Eq:KLIntegral} (see~\Cref{Lem:AvgKLRegularity}):
\begin{multline*}
    \frac{1}{2}\Wass_2^2(\rho_T^X, \nu^X)-\frac{1}{2} \Wass_2^2(\rho_0^X, \nu^X) \\
    \leq T^2 \, \KL(\rho_0^X \dvert \nu^X)-3\int_0^T(T-t)\KL(\rho_t^X\dvert\nu^X) \, \d t+\frac{M}{2}\int_0^T(T-t)\Wass_2^2(\rho_t^X, \nu^X) \, \d t\,.
\end{multline*}

\paragraph{Step 5: Regularity and approximation.}
We first impose~\Cref{asmp:init_regularity}: the initial relative density $q_0=\rho_0^X/\nu^X$ is continuously differentiable with bounded gradient and is bounded above and away from zero.
Under this assumption, \Cref{lem:warmness-propagation,lem:secondary_regularity} give the regularity of the densities, conditional moments, and characteristic flow used in Steps~1 and~2.
In Steps~2 and~4, we use an upper bound on centered second differences of the squared Wasserstein distance in place of an ordinary second time derivative; see~\Cref{lem:Wass22-FD,Lem:AvgKLRegularity}.
For the integration by parts in Step~3, we use the distributional derivative of the optimal transport map, as shown in the proof of~\Cref{Lem:KLBoundFormula}.
Together these arguments prove the integrated inequality under~\Cref{asmp:init_regularity}.
To remove~\Cref{asmp:init_regularity}, we construct in~\cref{Sec:ConsRegAppr} an approximating sequence $\rho_{0,n}^X\to\rho_0^X$ such that each $\rho_{0,n}^X$ satisfies~\Cref{asmp:init_regularity}; see \Cref{lem:ConsRegAppr}.
We apply the integrated inequality in~\Cref{Lem:AvgKLRegularity} to each $\rho_{0,n}^X$ and pass to the limit to obtain~\Cref{Lem:Wass-KL-integratedKey}; see~\Cref{Sec:Wass-KL-integratedKeyProof}.
This completes the proof sketch.
\end{proof}

%%%%%%%%%%%%%%%%%%%%%%%%
\subsection{Proofs of the main results}
\label{sec:MainProofs}
 
We use~\cref{Lem:Wass-KL-integratedKey} to establish the following result that leads to~\cref{thm:general-ti,thm:ConvRHMC_LC}.
We provide the proof of~\cref{lem:general-one-step} in~\cref{prf:lem:general-one-step}.

\begin{lemma}\label{lem:general-one-step}
Assume the setting of \cref{Lem:Wass-KL-integratedKey}, and assume \(w\) satisfies \cref{assump:integration-time}.
Let $T>0$, and write $\bar\rho^X=\E_{U \sim w}[\rho_{T \cdot U}^X]$.
Then the following inequalities hold:
\begin{itemize}
    \item If \(\nu^{X}\) satisfies \(\alpha\)-Talagrand's inequality and is \(M\)-semi-log-concave with \(0 \leq M < \alpha\), then
    \begin{equation*}
    \KL(\bar\rho^X\dvert\nu^X)+\frac{\Wass_2^2(\bar\rho^X,\nu^X)}{9T^2}
    \leq\left(1-\frac{1-M/\alpha}{9}\right)\KL(\rho_0^X\dvert\nu^X)
    +\frac{\Wass_2^2(\rho_0^X,\nu^X)}{9T^2} \,.
    \end{equation*}
    \item If \(\nu^{X}\) is log-concave, then
    \begin{equation*}
    \KL(\bar\rho^X\dvert\nu^X)+\frac{\Wass_2^2(\bar\rho^X,\nu^X)}{9T^2}
    \leq \frac{8}{9} \KL(\rho_0^X\dvert\nu^X)
    +\frac{\Wass_2^2(\rho_0^X,\nu^X)}{9T^2} \,.
    \end{equation*}
\end{itemize}
\end{lemma}

%%%%%%%%%%%%%%%
\subsubsection{Proof of \Cref{thm:general-ti}}\label{sec:proof-thm-ti}

\begin{proof}
We can show (see~\cref{lem:rhmc-admissibility}) that the iterates $\rho_k^X \in \P_{2,\ac,\fs}(\R^d)$ satisfy $\KL(\rho_k^X \dvert \nu^X)<\infty$ for all $k \ge 1$.
We apply the first setting of~\cref{lem:general-one-step} at iteration $k$ with $\rho_0^X \leftarrow \rho_{k-1}^X$, $\bar \rho^X \leftarrow \rho_k^X$, and $T^2 \leftarrow 4/(\alpha-M)$; and we discard the term $\Wass_2^2(\bar\rho^X,\nu^X) \ge 0$, to obtain:
\[
    \KL(\rho_k^X\dvert\nu^X)
    \leq\left(1-\frac{1-M/\alpha}{9}\right)
    \KL(\rho_{k-1}^X\dvert\nu^X)
    + \frac{(\alpha-M)}{36} \cdot \Wass_2^2(\rho_{k-1}^X,\nu^X) \,.
\]
Recall since $\nu^X$ satisfies $\alpha$-Talagrand's inequality,
$\Wass_2^2(\rho_{k-1}^X,\nu^X)
\leq\frac{2}{\alpha}\KL(\rho_{k-1}^X\dvert\nu^X)$.
Substituting this bound and simplifying, we obtain
\begin{align*}
    \KL(\rho_k^X\dvert\nu^X) &\leq \left(1-\frac{\alpha-M}{18\alpha}\right) \KL(\rho_{k-1}^X\dvert\nu^X) \,.
\end{align*}
Iterating this for $k=1,\ldots,K$ proves the desired bound in~\cref{eq:general-ti-convergence}.
\end{proof}

\subsubsection{Proof of \Cref{thm:ConvRHMC_LC}}\label{sec:proof-thm-lc}
\begin{proof}
We can show (see~\cref{lem:rhmc-admissibility}) that the iterates $\rho_k^X \in \P_{2,\ac,\fs}(\R^d)$ satisfy $\KL(\rho_k^X \dvert \nu^X)<\infty$ for all $k \ge 1$.
We apply the second setting of~\cref{lem:general-one-step} at iteration \(k\) with $\rho_0^X \leftarrow \rho_{k-1}^X$, $\bar \rho^X \leftarrow \rho_k^X$, and using \(T_k^2=(9/8)T_{k-1}^2\), to obtain:
\begin{align*}
\KL(\rho_k^X\dvert\nu^X)+\frac{\Wass_2^2(\rho_k^X,\nu^X)}{9T_k^2}
&\leq \frac89\KL(\rho_{k-1}^X\dvert\nu^X)+\frac{\Wass_2^2(\rho_{k-1}^X,\nu^X)}{9T_k^2}\\
&=\frac89\left(\KL(\rho_{k-1}^X\dvert\nu^X)+\frac{\Wass_2^2(\rho_{k-1}^X,\nu^X)}{9T_{k-1}^2}\right) \,.
\end{align*}
Iterating this for \(k=1,\ldots,K\), and substituting \(T_0^2 = 1\) and \(T_K^2=(9/8)^K\), proves~\cref{eq:CleanConvergenceWLC}.
\end{proof}

%%%%%%%%%%%%%%%%%%%%%%%%
\subsubsection{Proof of \cref{lem:general-one-step}}
\label{prf:lem:general-one-step}

\begin{proof}
In this proof, we set $\eta=M/\alpha$ in the first setting when $\nu^X$ is $M$-semi-log-concave and satisfies $\alpha$-Talagrand's inequality; and we set  $\eta=0$ in the second setting when $\nu^X$ is log-concave, so that we can handle both settings simultaneously.

Fix $T > 0$ and $u \ge 0$.
By applying~\cref{Eq:KLIntegral} from~\cref{Lem:Wass-KL-integratedKey} at time $Tu$ 
and changing variable $t=Ts$ (and by Talagrand's inequality in the first setting which gives $\Wass_2^2(\rho_t^X,\nu^X)\leq (2/\alpha)\KL(\rho_t^X\dvert\nu^X)$), we obtain in either setting:
\begin{multline*}
    \Wass_2^2(\rho_{Tu}^X,\nu^X)
    +(6-2\eta)T^2\int_0^u (u-s)\KL(\rho_{Ts}^X\dvert\nu^X)\,\ds
    \leq \Wass_2^2(\rho_0^X,\nu^X)+2T^2u^2\KL(\rho_0^X\dvert\nu^X) \,.
\end{multline*}
Integrating against $w(u)$ over $u \in [0,\infty)$ and exchanging the order of integration, we obtain
\begin{align*}
    \int_0^\infty w(u)\Wass_2^2(\rho_{Tu}^X,\nu^X)\du
    +(6-2\eta)T^2\int_0^\infty A(s)\KL(\rho_{Ts}^X\dvert\nu^X)\ds 
    \leq \Wass_2^2(\rho_0^X,\nu^X)+2T^2\KL(\rho_0^X\dvert\nu^X) \,,
\end{align*}
where we have used $\E[U^2] = 1$ and defined $A(s)=\int_s^\infty (u-s)w(u)\,\du$.  

We can show (see~\cref{lem:Cw-bound-log-concave-non-inc}) that $A(s)\leq \tfrac32 w(s)$ for all $s \ge 0$, which implies $(6-2\eta)A(s)\leq 9w(s)$ since $\eta \ge 0$.
By \cref{Lem:DescentKL}, we know $\KL(\rho_{Ts}^X\dvert\nu^X)-\KL(\rho_0^X\dvert\nu^X) \le 0$.
Therefore,
\begin{align*}
    (6-2\eta) \, A(s) \,\bigl(\KL(\rho_{Ts}^X\dvert\nu^X)-\KL(\rho_0^X\dvert\nu^X)\bigr)
    \geq 9w(s) \,\bigl(\KL(\rho_{Ts}^X\dvert\nu^X)-\KL(\rho_0^X\dvert\nu^X)\bigr) \,.
\end{align*}
Integrating over $s \in [0,\infty)$ and using $\int_0^\infty A(s)\,\ds= \frac{1}{2} \E[U^2] = \frac{1}{2}$ and $\int_0^\infty w(s)\,\ds=1$, we obtain
\begin{align*}
    (6-2\eta)\int_0^\infty A(s) \,\KL(\rho_{Ts}^X\dvert\nu^X)\,\ds 
    \geq 9\int_0^\infty w(s)\KL(\rho_{Ts}^X\dvert\nu^X)\,\ds
    -(6+\eta)\KL(\rho_0^X\dvert\nu^X) \,.
\end{align*}

We substitute this lower bound into the integral inequality above. 
Since $\bar\rho^X=\int_0^\infty w(s)\rho_{Ts}^X\,\ds$, by the convexity of the KL divergence and the squared $\Wass_2$-distance under mixtures, we obtain:
\begin{align*}
    \Wass_2^2(\bar\rho^X,\nu^X) + 9T^2 \, \KL(\bar\rho^X\dvert\nu^X) 
    &\leq \int_0^\infty w(s)\bigl(\Wass_2^2(\rho_{Ts}^X,\nu^X) + 9T^2 \, \KL(\rho_{Ts}^X\dvert\nu^X)\bigr)\,\ds \\
    &\leq \Wass_2^2(\rho_0^X,\nu^X)+(8+\eta) \, T^2\KL(\rho_0^X\dvert\nu^X) \,.
\end{align*}
Dividing both sides by $9T^2$ proves the desired bound in either setting.
\end{proof}

%%%%%%%%%%%%%%%%%
\section{Connections to Hamiltonian flow for optimization}\label{app:HMCandHFopt}

We explain the connection between our results for sampling and the results of~\cite{wang26} for accelerated optimization.
Both analyses use an averaging property of Hamiltonian flow under convexity. 
To see the connection directly, we consider the triangular integration time, for which the two methods use the same averaging weights. 
However, our results for sampling in this paper apply to a broader class of integration-time densities.

Consider a convex function $f \colon \R^d \to \R$ with a minimizer $x^\star \in \arg\min_{x \in \R^d} f(x)$. 
In~\cite{wang26}, building on earlier work of~\cite{teel2019first,fu2025hamiltoniandescentalgorithmsoptimization,wang2025frictionlesshamiltoniandescentcoordinate}, we developed an optimization algorithm for minimizing $f$ by simulating the Hamiltonian dynamics and periodically resetting the velocity to zero.
In contrast, recall in \HMC for sampling from $\nu^X \propto e^{-f}$, we simulate the Hamiltonian dynamics and periodically refresh the velocity to be an independent Gaussian random variable.
Throughout this section we take $f$ to be convex. In the sampling setting, this means that $\nu^X$ is log-concave and that $\rho\mapsto\KL(\rho\dvert\nu^X)$ is geodesically convex in $(\P_2(\R^d),\Wass_2)$.
We can capture both optimization and sampling settings by using $q_t$ for the position variable or its law, $\mathcal F$ for the quantity to be optimized, $q^\star$ to be the minimizer, and $\mathsf{dist}$ for the distance.
Then we have the following correspondence:
\[
\begin{array}{c|ccccc}
 & \text{domain} & ~q_t~ & ~q^\star~ & \mathcal F(q_t) & \mathsf{dist}(q_t,q^\star) \\ \hline
\text{Optimization} & \R^d & X_t & x^\star & f(X_t)-f(x^\star) & \|X_t-x^\star\| \\
\text{Sampling} & \PtwoacfsRd & \rho_t^X & \nu^X & \KL(\rho_t^X\dvert\nu^X) & \Wass_2(\rho_t^X,\nu^X)
\end{array}
\]
In optimization, the Hamiltonian value $f(X_t) + \frac{1}{2} \|Y_t\|^2$ is conserved, and the zero initial velocity $Y_0 = 0$ results in a descent method: $f(X_t)\leq f(X_0)$. 
In sampling, the joint $\KL$ divergence $\KL(\rho_t^{XY} \dvert \nu^{XY})$ is conserved, and the Gaussian initial velocity results in a descent guarantee for the position marginal, as shown in \cref{Lem:DescentKL}. 

To achieve accelerated rate,~\cite{wang26} proposes using long integration time and outputting the average of the position trajectory as the next iterate.
In \RHMC, we use a long random integration time, and output the last position at that random time; its law is the average of the laws of the position marginals, so we also follow the averaging strategy but in the space of distributions.
The correspondence is especially clear in the bound on the second derivative. The optimization inequality from~\cite[Lemma~2]{wang26} and the sampling inequality~\eqref{Eq:DiffIneqW2} (with $M=0$) can both be written formally as
\[
    \frac12\frac{\d^2}{\dt^2}\mathsf{dist}^2(q_t,q^\star)
    \leq 2\mathcal F(q_0)-3\mathcal F(q_t) \,.
\]
Integrating twice from $t=0$ to $T$ produces the inequality:
\[
    \frac{\mathsf{dist}^2(q_T,q^\star)}{3T^2}
    +\int_0^T \frac{2(T-t)}{T^2}\,\mathcal F(q_t)\dt
    \leq \frac{\mathsf{dist}^2(q_0,q^\star)}{3T^2}+\frac23\mathcal F(q_0) \,,
\]
which is~\cref{Lem:Wass-KL-integratedKey} in the sampling setting with $M = 0$.
Thus, the triangular density is the normalized weight in the integrated inequality.

Let $\bar q_T=\int_0^T \frac{2(T-t)}{T^2} q_t \, \dt$. 
In optimization, $\bar q_T$ is the next iterate of the trajectory-averaging method of~\cite{wang26}. 
In sampling, $\bar q_T$ is the law of the next \RHMC iterate when the integration time is drawn from the triangular distribution on $[0,T]$. 
(Here $T$ is the endpoint of the support; in~\RHMC with a triangular clock and scale parameter $T_k$, the corresponding endpoint is $\sqrt{6}\,T_k$.)
By Jensen's inequality, $\mathcal F(\bar q_T)$ is at most the weighted integral of $\mathcal F(q_t)$ above. 
This uses the ordinary convexity of $f$ in optimization, and the convexity of $\KL$ divergence under mixtures in sampling.

If $\mathcal{F}$ satisfies the condition $\mathcal F(q)\geq \frac{\alpha}{2} \mathsf{dist}^2(q,q^\star)$, then the above gives the one-step bound:
\[
    \mathcal F(\bar q_T)\leq \frac23\left(1+\frac{1}{\alpha T^2}\right)\mathcal F(q_0) \,,
\]
after dropping the nonnegative endpoint-distance term on the left.
This condition is the \textit{quadratic growth} condition for $f$ in optimization, and it is the \textit{Talagrand's inequality} for $\nu^X \propto \exp(-f)$ in sampling. 
In both settings, taking $T=2/\sqrt{\alpha}$ gives a fixed contraction, which when iterated results in an accelerated exponential convergence rate. 
Without the quadratic growth assumption, we keep the squared distance term so that we can use it in the next iteration. The integrated inequality controls this distance at a fixed endpoint. We average the inequality over the endpoints using the same weights that define the next iterate. By convexity, the resulting average of squared distances bounds the squared distance of the next iterate. This leads to an accelerated rate with increasing integration-time scales; see~\cref{prf:lem:general-one-step} for the sampling argument.

%%%%%%%%%%%%
\section{Discussion}\label{sec:discussion}

In this work, we prove accelerated mixing time guarantees for idealized \RHMC in two settings: when the target distribution is $M$-semi-log-concave and satisfies $\alpha$-Talagrand's inequality with $M<\alpha$, and when it is log-concave without assuming Talagrand's inequality.
The resulting bounds on total Hamiltonian integration time improve on the corresponding bounds for overdamped Langevin dynamics in the log-concave cases discussed in the introduction and match the rates we expect from accelerated convex optimization.
Our results hold for any integration time distribution with log-concave and non-increasing density.
For a log-concave target, the key estimate in~\cref{Lem:Wass-KL-integratedKey} bounds a weighted average of the position \textsf{KL} divergence along Hamiltonian flow in terms of the initial \textsf{KL} divergence and squared Wasserstein distance.
This property is motivated by our prior work on Hamiltonian dynamics-based optimization~\cite{wang26}.
Our results complement the recent results on the accelerated mixing time guarantees for underdamped Langevin dynamics~\cite{lu2026sharphypocoerciveentropydecay,li2026spacetime} and \RHMC~\cite{monmarche2026entropicconvergencepiecewisedeterministic}, which proceed via different analysis techniques.

Our results assume that we can simulate the Hamiltonian flow exactly. 
Recent work has made progress on implementable accelerated samplers; see, e.g., \cite{altschuler2025shifted,lu2026windowed,chen2026exactsimulation,chen2026proximalbps,yuan2026reflected,bou2026persistence}.
In particular,~\cite{bou2026breaking} adapts our entropy averaging argument to the Verlet proposal and uses it to prove total variation mixing for the Metropolis-adjusted chain from a warm start. For log-concave targets satisfying Talagrand’s inequality, the gradient cost scales with the square-root of the condition number when the Hessian-Lipschitz term does not dominate.
An open question is whether a numerical implementation can retain the accelerated rate of our increasing-time schedule without assuming Talagrand's inequality.
It would also be interesting to investigate guarantees for adaptive variants of \HMC, such as the No-U-Turn Sampler (NUTS)~\cite{hoffman2014no} which is widely used in practice, and whose theory is still being developed~\cite{bou2024mixing,oberdorster2025accelerated,bou2026within,gruffaz2026theoretical}.

%%%%%%%%%%%
\subsubsection*{Acknowledgements}

We thank Jason Altschuler, Nawaf Bou-Rabee, Andreas Eberle, Jianfeng Lu, and Pierre Monmarch\'e for insightful comments and feedback on an earlier version of this paper. 

\smallskip
\noindent
\textbf{Funding:}
AW, SM, XW are supported by NSF awards CCF-2403391 and CAREER CCF-2443097.
VS is supported by a MathWorks Engineering Fellowship.

\smallskip
\noindent
\textbf{AI usage disclosure:}
The main ideas of this work arose from adapting our work on optimization~\cite{wang26} to sampling. 
ChatGPT helped develop some of the statements and proofs, in particular for addressing regularity issues. 
The authors retain responsibility for the correctness of the results and their presentation.

%%%%%%%%
\appendix

%%%%%%%%%%
\section{Properties of the Hamiltonian flow}
\label[appendix]{apdx:HF_review}

We review the properties of the Hamiltonian flow, then derive the equations for the conditional moments of distributions evolved by the Hamiltonian flow.
Throughout, we retain \cref{assump:target-regularity}: $f\in C^2(\R^d)$ is $L$-smooth for some $L \in (0,\infty)$, and $\nu^X\propto e^{-f}$ belongs to $\PtwoacfsRd$. Recall from \cref{Eq:HamDef} the Hamiltonian of interest, and the corresponding Hamiltonian vector field:
$$H(x,y)=f(x)+\frac12\|y\|^2 \,, \qquad b(x, y) \deq (y, -\nabla f(x)) \,.$$ 
The Hamiltonian flow \eqref{eq:HamFlow} can be written as $(\dot{X}_t, \dot Y_t)=b(X_t, Y_t) = (Y_t, -\nabla f(X_t))$.

\paragraph{Hamiltonian flow map.}
Whenever \eqref{eq:HamFlow} admits a continuously differentiable solution for all
\(t\in\R\) and \((X_0,Y_0)\in\R^{2d}\), we denote the corresponding Hamiltonian flow map by
$$\Psi\colon\R\times\R^{2d}\to\R^{2d} \,,
\qquad(t,X_0,Y_0)\mapsto \Psi(t,X_0,Y_0) \equiv \Psi_t(X_0,Y_0) \,.
$$
That is, \(\Psi_t(X_0,Y_0)=(X_t,Y_t)\) is the solution of
\eqref{eq:HamFlow} at time \(t\) with initial condition \((X_0,Y_0)\). 
We also define the marginal flow maps
$$\Psi_t^X(X_0,Y_0)=\Pi_X(\Psi_t(X_0,Y_0)) \,,\qquad\Psi_t^Y(X_0,Y_0)=\Pi_Y(\Psi_t(X_0,Y_0)) \,,
$$
where \(\Pi_X \colon \R^{2d} \to \R^d\) and \(\Pi_Y  \colon \R^{2d} \to \R^d \) denote projections onto the first and second variables, respectively: $\Pi_X(x,y) = x$ and $\Pi_Y(x,y) = y$.

\begin{lemma}\label{lem:Hamiltonian_smoothness}
If $f$ is $L$-smooth, then $b(x, y) = (y, -\nabla f(x))$ is $(1+L)$-Lipschitz, and~\eqref{eq:HamFlow} has a unique solution for every $t\in\R$ and every initial condition.
\end{lemma}
\begin{proof}
For any $z=(x,y)$ and $\tilde z=(\tilde x,\tilde y)$, we can bound:
\begin{align*}
\|b(z)-b(\tilde z)\|
\leq \|y-\tilde y\|+L\|x-\tilde x\|
\leq (1+L)\|z-\tilde z\| \,.
\end{align*}
Global existence and uniqueness follow from the Picard--Lindel\"of theorem for globally Lipschitz vector fields.
\end{proof}

In particular, $\Psi_0=\operatorname{Id}$, $\Psi_{t+s}=\Psi_t\circ\Psi_s$, and $\Psi_t^{-1}=\Psi_{-t}$, by standard properties of autonomous ODEs.
The assumption $f\in C^2$ also implies that $\Psi_t$ is a $C^1$ diffeomorphism.

%%%%%%%%%%%%%%%%
\subsection{Deterministic properties of the Hamiltonian flow}
\label[appendix]{subapp:HF-deterministic-properties}

We recall some properties of the Hamiltonian flow.
First, the Hamiltonian value is conserved.

\begin{lemma}\label{lem:energy_conservation}
Let $(X_t,Y_t)$ solve~\eqref{eq:HamFlow} from $(X_0,Y_0)\in\R^{2d}$.
Then $H(X_t,Y_t)=H(X_0,Y_0)$ for all $t\in\R$.
\end{lemma}
\begin{proof}
By the chain rule,
\[
\frac{d}{dt} H(X_t,Y_t) = \langle \dot X_t, \nabla f(X_t) \rangle + \langle \dot Y_t, Y_t \rangle 
= \langle Y_t, \nabla f(X_t)\rangle-\langle Y_t, \nabla f(X_t)\rangle = 0 \,.
\]
\end{proof}

We also recall the Hamiltonian flow is Lebesgue volume-preserving.
See e.g.~\cite[Section 3.5]{Leimkuhler_Reich_2005} or~\cite[Section B.1]{wang26} for a proof.

\begin{lemma}\label{lem:volume_conservation}
For every $t\in\R$ and $z\in\R^{2d}$, $\det\nabla\Psi_t(z)=1$.
Consequently, $\operatorname{vol}(\Psi_t(A))=\operatorname{vol}(A)$ for every measurable $A\subseteq\R^{2d}$.
\end{lemma}

Next, we establish that the Hamiltonian flow map is Lipschitz continuous.

\begin{lemma}\label{lem:biLip_flow_map}
If $f$ is $L$-smooth, then for all $t\in\R$ and $z,\tilde z\in\R^{2d}$,
\[
e^{-(1+L)|t|}\|z-\tilde z\|
\leq\|\Psi_t(z)-\Psi_t(\tilde z)\|
\leq e^{(1+L)|t|}\|z-\tilde z\| \,.
\]
Furthermore, $\|\nabla\Psi_t(z)\|_{\op}\leq e^{(1+L)|t|}$.
\end{lemma}
\begin{proof}
We show the estimate for $t\geq 0$; the case $t\leq 0$ follows by applying the same argument to the backward flow. For $t\geq0$, the integral form of~\eqref{eq:HamFlow} gives
\begin{align*}
\Psi_t(z)-\Psi_t(\tilde z)
&=z-\tilde z+\int_0^t\bigl(b(\Psi_s(z))-b(\Psi_s(\tilde z))\bigr)\ds \,,\\
\|\Psi_t(z)-\Psi_t(\tilde z)\|
&\leq\|z-\tilde z\|+(1+L)\int_0^t\|\Psi_s(z)-\Psi_s(\tilde z)\|\ds ~\leq~ e^{(1+L)t}\|z-\tilde z\| \,.
\end{align*}
The first inequality holds because $b$ is $(1+L)$-Lipschitz when $f$ is $L$-smooth (\cref{lem:Hamiltonian_smoothness}). 
The second bound follows from Gr\"onwall's inequality. Since the backward flow \(s \mapsto \Psi_{-s}\) is generated by the vector field $-b$ which is also $(1+L)$-Lipschitz, the same argument applied to \(u,v\in\R^{2d}\) and \(t\geq 0\) yields
$$\|\Psi_{-t}(u)-\Psi_{-t}(v)\|\leq e^{(1+L)t}\|u-v\| \,.$$
Taking \(u=\Psi_t(z_0)\) and \(v=\Psi_t(\widetilde z_0)\), and using \(\Psi_{-t}=(\Psi_t)^{-1}\), we obtain the first claim:
$$\|z_0-\widetilde z_0\|=\|\Psi_{-t}(\Psi_t(z_0))-\Psi_{-t}(\Psi_t(\tilde{z}_0))\|
\leq e^{(1+L)t}\|\Psi_t(z_0)-\Psi_t(\tilde{z}_0)\| \,.$$
Since the flow map $z\mapsto \Psi_t(z)$ is continuously differentiable with respect to $z$, the Lipschitz bound above guarantees $\|\nabla \Psi_t(z)\|_{\op}\leq e^{(1+L)|t|}$. 
\end{proof}

\subsection{Distributional properties of the Hamiltonian flow}
\label[appendix]{subapp:HF-distributional-properties}

Let $(X_0,Y_0)\sim\rho_0^{XY}\in\P_{2,\ac,\fs}(\R^{2d})$ and $(X_t,Y_t)=\Psi_t(X_0,Y_0)$.
Then $\rho_t^{XY}=(\Psi_t)_\#\rho_0^{XY}$, and $\rho_t^X$ denotes its position marginal.

\subsubsection{Conservation of \textsf{KL} divergence to the joint target distribution}
\label[appendix]{Sec:HamFlowConsvJointKLProof}

\begin{proof}[Proof of~\Cref{Lem:HamFlowConsvJointKL}]
By volume conservation (\Cref{lem:volume_conservation}) and change of variables, and by energy conservation (\Cref{lem:energy_conservation}) with $\nu^{XY}\propto e^{-H}$, respectively,
\[
\rho_t^{XY}(\Psi_t(z))=\rho_0^{XY}(z),\qquad
\nu^{XY}(\Psi_t(z))=\nu^{XY}(z) \,,
\]
for almost every $z$. Since \(\Psi_t:\R^{2d}\to\R^{2d}\) is a diffeomorphism with inverse \(\Psi_{-t}\):
\begin{align*}
\KL(\rho_0^{XY}\dvert\nu^{XY})&=\int_{\R^{2d}}\rho_0^{XY}(z)
\log\frac{\rho_0^{XY}(z)}{\nu^{XY}(z)}\dz\\
&=\int_{\R^{2d}}\rho_t^{XY}(\Psi_t(z))
\log\frac{\rho_t^{XY}(\Psi_t(z))}{\nu^{XY}(\Psi_t(z))}\dz=\KL(\rho_t^{XY}\dvert\nu^{XY}) \,.
\end{align*}
\end{proof}

\subsubsection{Descent property of \textsf{KL} divergence in the $X$-marginal}
\label[appendix]{Sec:DescentKLProof}

\begin{proof}[Proof of~\Cref{Lem:DescentKL}]
Since $\nu^{XY}=\nu^X\otimes\gamma_{d}$, the chain rule for \textsf{KL} divergence~\citep[Theorem~2.15]{polyanskiywu2024information} gives
\[
\KL(\rho_t^{XY}\dvert\nu^{XY})
=\KL(\rho_t^X\dvert\nu^X)
+\E_{x\sim\rho_t^X}\bigl[\KL(\rho_t^{Y\mid X=x}\dvert\gamma_{d})\bigr] \,. 
\]
By~\Cref{Lem:HamFlowConsvJointKL}, the joint \textsf{KL} divergence is conserved along the flow. 
Initially, $\rho_0^{XY}=\rho_0^X\otimes\gamma_{d}$, so the two joint distributions have the same velocity distribution and therefore,
\[
\KL(\rho_t^{XY}\dvert\nu^{XY})
=\KL(\rho_0^{XY}\dvert\nu^{XY})
=\KL(\rho_0^X\dvert\nu^X)<\infty \,.
\]
Substituting this identity into the chain rule and rearranging gives
\[
\KL(\rho_t^X\dvert\nu^X)
=\KL(\rho_0^X\dvert\nu^X)
-\E_{x\sim\rho_t^X}\bigl[\KL(\rho_t^{Y\mid X=x}\dvert\gamma_{d})\bigr] \,.
\]
The conditional \textsf{KL} divergence is nonnegative, which gives the claimed upper bound.
\end{proof}

\subsubsection{Continuity of \textsf{KL} divergence}
\label[appendix]{Sec:KLContinuousProof}

\begin{lemma}\label{Lem:KLContinuous}
Let $\rho_0^{XY}\in\P_{2,\ac,\fs}(\R^{2d})$ satisfy $\KL(\rho_0^{XY}\dvert\nu^{XY})<\infty$.
Along~\eqref{eq:HamFlow} from $\rho_0^{XY}$, the map $t\mapsto\KL(\rho_t^X\dvert\nu^X)$ is continuous on $\R$.
\end{lemma}
\begin{proof}
Fix $t \in \R$ and a sequence $t_n\to t$. Continuity of $\Psi_t$ in time and dominated convergence give, for every bounded continuous test function $\varphi\colon\R^{2d}\to\R$,
\[
\int_{\R^{2d}}\varphi(z)\rho_{t_n}^{XY}(z)\dz
=\int_{\R^{2d}}\varphi(\Psi_{t_n}(z))\rho_0^{XY}(z)\dz
\to \int_{\R^{2d}}\varphi(\Psi_{t}(z))\rho_0^{XY}(z)\dz=\int_{\R^{2d}}\varphi(z)\rho_t^{XY}(z)\dz \,.
\]
Thus, $\rho_{t_n}^{XY}\to\rho_t^{XY}$ weakly, and the same also holds for the position marginals and their products with $\gamma_{d}$.

For every $s\in\R$, the chain rule with $\nu^{XY}=\nu^X\otimes\gamma_{d}$ and conservation of the joint \textsf{KL} divergence (\Cref{Lem:HamFlowConsvJointKL}) give
\begin{align}
\KL(\rho_s^X\dvert\nu^X)+\KL(\rho_s^{XY}\dvert\rho_s^X\otimes\gamma_{d})
&=\KL(\rho_s^X\dvert\nu^X)
+\E_{x\sim\rho_s^X}\bigl[\KL(\rho_s^{Y\mid X=x}\dvert\gamma_{d})\bigr]\notag\\
&=\KL(\rho_s^{XY}\dvert\nu^{XY})
=\KL(\rho_0^{XY}\dvert\nu^{XY})<\infty \,.
\label{Eq:KLContCalc1}
\end{align}
By the lower semicontinuity of \textsf{KL} divergence~\citep[Theorem~4.9]{polyanskiywu2024information}, we have
\begin{align}
\KL(\rho_t^X\dvert\nu^X)
&\leq\liminf_{n\to\infty}\KL(\rho_{t_n}^X\dvert\nu^X) \,,
\label{Eq:KLContCalc2}\\
\KL(\rho_t^{XY}\dvert\rho_t^X\otimes\gamma_{d})
&\leq\liminf_{n\to\infty}\KL(\rho_{t_n}^{XY}\dvert\rho_{t_n}^X\otimes\gamma_{d}) \,.
\label{Eq:KLContCalc0}
\end{align}
Let $C\deq\KL(\rho_0^{XY}\dvert\nu^{XY})=\KL(\rho_t^{XY}\dvert\nu^{XY})$. 
By~\eqref{Eq:KLContCalc1}, equation~\eqref{Eq:KLContCalc0} becomes
\begin{align*}
C-\KL(\rho_t^X\dvert\nu^X)
&\leq\liminf_{n\to\infty}\bigl(C-\KL(\rho_{t_n}^X\dvert\nu^X)\bigr)
=C-\limsup_{n\to\infty}\KL(\rho_{t_n}^X\dvert\nu^X) \,.
\end{align*}
Subtracting $C$ and reversing the sign gives $\limsup_{n\to\infty}\KL(\rho_{t_n}^X\dvert\nu^X)\leq\KL(\rho_t^X\dvert\nu^X)$.
Together with~\eqref{Eq:KLContCalc2}, this proves continuity.
\end{proof}

%%%%%%%%%%%%%%%%%%%%%%%%%%
\subsubsection{Finiteness of second moment}
\label[appendix]{Sec:second-moment-finitenessProof}

\begin{lemma}\label{lem:second-moment-finiteness}
Let $(X_t,Y_t)\sim\rho_t^{XY}$ solve~\eqref{eq:HamFlow} from $\rho_0^{XY}\in\P_{2,\ac,\fs}(\R^{2d})$.
Define $C_L\deq 2L+3$. For every $t\in\R$,
\[
\E\bigl[\|(X_t,Y_t)\|^2\bigr]
\leq e^{C_L|t|}\E\bigl[\|(X_0,Y_0)\|^2\bigr]
+\frac{\|\nabla f(0)\|^2}{C_L}\bigl(e^{C_L|t|}-1\bigr) \,.
\]
In particular, $\rho_t^{XY}\in\P_{2,\ac,\fs}(\R^{2d})$ and $\rho_t^X\in\PtwoacfsRd$, and their second moments are bounded uniformly on compact time intervals.
\end{lemma}
\begin{proof}
Write $Z_t=(X_t,Y_t)$ and recall that $b(x,y)=(y,-\nabla f(x))$ is the Hamiltonian vector field. By~\cref{lem:Hamiltonian_smoothness}, $b$ is $(1+L)$-Lipschitz. In particular, for every $z\in\R^{2d}$,
\[
\|b(z)\|\leq(1+L)\|z\|+\|\nabla f(0)\| \,.
\]
Differentiating $\|Z_t\|^2$ along the Hamiltonian flow and using $2uv\leq u^2+v^2$ gives
\begin{align*}
\frac{\mathrm d}{\mathrm dt}\|Z_t\|^2
\,=\, 2\langle Z_t,b(Z_t)\rangle
\,\leq\, 2(1+L)\|Z_t\|^2+2\|Z_t\|\|\nabla f(0)\|
\,\leq\, C_L\|Z_t\|^2+\|\nabla f(0)\|^2 \,.
\end{align*}
Applying Gr\"onwall's inequality and then taking expectation gives the stated estimate for $t\geq0$. 
For $t<0$, we apply the same argument to the backward flow, since the bound on $\|b(z)\|$ also holds for $-b$. 
The estimate is uniform on compact time intervals and also bounds the second moment of $X_t$.
It remains to verify that both distributions have strictly positive densities. 
The joint distribution $\rho_t^{XY}=(\Psi_t)_\#\rho_0^{XY}$ has a strictly positive density by the change-of-variables formula, since $\Psi_t$ is a diffeomorphism. 
Integrating $\rho_t^{XY}$ over the $Y$ variable gives a strictly positive density for $\rho_t^X$.
\end{proof}

\subsection{Moments along Hamiltonian flow}
\label[appendix]{Sec:MomentsHamFlow}

Along the Hamiltonian flow from $\rho_0^{XY}\in\P_{2,\ac,\fs}(\R^{2d})$, recall the conditional moments for \(Y_{t}\) given \(X_{t} = x\) for $\rho_t^X$-almost every $x$ from \cref{eq:def-cond-moments}.
\begin{equation*}
u_t(x)\deq\E[Y_t\mid X_t=x] \,,\qquad
M_t(x)\deq\E[Y_tY_t^\top\mid X_t=x] \,,\qquad
\Sigma_t(x)\deq M_t(x)-u_t(x)u_t(x)^\top \,.
\end{equation*}
These moments are finite for $\rho_t^X$-almost every $x$, by~\Cref{lem:second-moment-finiteness}.
Whenever the derivatives are justified, we define the \emph{acceleration field} by
\begin{equation}\label{Eq:DefAccField}
a_t(x)\deq\partial_tu_t(x)+\nabla u_t(x)u_t(x) \,.
\end{equation}
Thus, if $\dot\Phi_t(x)=u_t(\Phi_t(x))$, then $\ddot\Phi_t(x)=a_t(\Phi_t(x))$.

Additionally, by Brenier's theorem, the \emph{optimal transport map} $R_t \colon \R^d \to \R^d$ from $\rho_t^X$ to $\nu^X$ exists and satisfies
$$\Wass_2^2(\rho_t^X,\nu^X)=\int_{\R^d}\|x-R_t(x)\|^2\rho_t^X(x)\dx \,.$$
We define the \emph{displacement map}
\begin{equation*}
v_t(x)\deq x-R_t(x) \,.
\end{equation*}

%%%%%%%%%%%%%%
\subsubsection{Continuity and momentum equations}

We recall that an equation holds in the sense of distributions if it holds after testing against smooth compactly supported test functions.

\begin{lemma}\label{lem:continuity-momentum-equation}
Along~\eqref{eq:HamFlow} from $\rho_0^{XY}\in\P_{2,\ac,\fs}(\R^{2d})$, the following equations hold in the sense of distributions on $\R^d\times(0,T)$ for every $T>0$:
\begin{align*}
\partial_t\rho_t^X+\nabla\cdot(\rho_t^Xu_t)&=0 \,,\\
\partial_t(\rho_t^Xu_t)+\nabla\cdot(\rho_t^XM_t)+\rho_t^X\nabla f&=0 \,.
\end{align*}
\end{lemma}
\begin{proof}
Fix $T>0$. We first prove the continuity equation. Let $\varphi\in C_c^\infty(\R^d\times(0,T))$. Along each trajectory of~\eqref{eq:HamFlow}, the time integral of $\frac{\mathrm d}{\mathrm dt}\varphi(X_t,t)$ is zero, because $\varphi$ vanishes near $t=0$ and $t=T$. Taking the expectation over the initial condition and applying the chain rule yields
\begin{align*}
0
&=\E_{(X_0,Y_0)\sim\rho_0^{XY}}\left[\int_0^T\frac{\mathrm d}{\mathrm dt}\varphi(X_t,t)\dt\right]\\
&=\int_0^T\E_{(x,y)\sim\rho_t^{XY}}\bigl[\partial_t\varphi(x,t)+\langle\nabla_x\varphi(x,t),y\rangle\bigr]\dt \,.
\end{align*}
Here and below, we can interchange time integration and expectation by Fubini's theorem, since the test functions and their derivatives are bounded, $\nabla f$ grows at most linearly, and the second moment is uniformly bounded on $[0,T]$ by~\Cref{lem:second-moment-finiteness}.
Conditioning on $X_t=x$ replaces $y$ in the linear term by $\E[Y_t\mid X_t=x]=u_t(x)$. Thus,
\[
0=\int_0^T\int_{\R^d}
\bigl(\partial_t\varphi(x,t)+\langle\nabla_x\varphi(x,t),u_t(x)\rangle\bigr)\rho_t^X(x)\dx\dt \,.
\]
This is the continuity equation in the sense of distributions.

To prove the momentum equation, let $\psi\in C_c^\infty(\R^d\times(0,T);\R^d)$ and apply the same argument to $\langle\psi(X_t,t),Y_t\rangle$. Differentiating both factors gives
\begin{align*}
0=\int_0^T\E_{(x,y)\sim\rho_t^{XY}}\bigl[\langle\partial_t\psi(x,t),y\rangle
+\langle\nabla_x\psi(x,t)y,y\rangle -\langle\psi(x,t),\nabla f(x)\rangle\bigr]\dt \,.
\end{align*}
Since $\langle\nabla_x\psi(x,t)y,y\rangle=\langle\nabla_x\psi(x,t),yy^\top\rangle_{\mathsf F}$, conditioning on $X_t=x$ replaces $y$ and $yy^\top$ by $u_t(x)$ and $M_t(x)$, respectively. Therefore,
\begin{align*}
0=\int_0^T\int_{\R^d}\bigl(\langle\partial_t\psi(x,t),u_t(x)\rangle
+\langle\nabla_x\psi(x,t),M_t(x)\rangle_{\mathsf F} -\langle\psi(x,t),\nabla f(x)\rangle\bigr) \, \rho_t^X(x)\dx\dt \,.
\end{align*}
This is the momentum equation in the sense of distributions.
\end{proof}

%%%%%%%%%%%%%%%%%%%%%%%%%%%%%%

\section{Properties of the Hamiltonian flow under regular initialization}
\label[appendix]{apdx:primary_to_secondary}

In this section, we establish the regularity of the position marginal along the Hamiltonian flow when the initial distribution satisfies the assumption below, which is inspired by~\cite[Definition~2.5]{lu2026sharphypocoerciveentropydecay}. 
We use this regularity in~\cref{apdx:Wass22_proof} to justify the Wasserstein estimates and the integration by parts. 
In~\cref{apdx:approximation-proof}, we remove the assumption on the initial distribution by approximation. 
Throughout this appendix, we retain~\cref{assump:target-regularity}.

\begin{assumption}\label{asmp:init_regularity}
Let $\rho_0^X\in\PtwoacfsRd$ and $q_0\deq\rho_0^X/\nu^X$.
We assume $q_0$ is $C^1$ with bounded first derivative, and there exist $0<\zeta\leq\xi<\infty$ such that $\zeta\leq q_0(x)\leq\xi$ for every $x\in\R^d$.
\end{assumption}

Throughout this appendix, $(X_t,Y_t)$ follows the Hamiltonian flow~\eqref{eq:HamFlow} from $(X_0,Y_0)\sim\rho_0^X\otimes\gamma_{d}$, with joint distribution $(X_t,Y_t) \sim \rho_t^{XY}$ and position marginal $X_t \sim \rho_t^X$.

We first show in~\cref{lem:warmness-propagation} that the upper and lower bounds on $q_0$ are propagated along the flow. 
This gives a positive lower bound on the denominator in the formulas for the conditional velocity moments. 
In~\cref{lem:secondary_regularity}, we differentiate these formulas to control the conditional mean velocity $u_t$ and the acceleration $a_t$, and show the integrability properties needed in the integration by parts in~\cref{apdx:Wass22_proof}. 
Finally, we use these bounds to construct the characteristic flow for $\rho_t^X$ and justify its first and second time difference quotients in $L^2(\rho_0^X)$, which are further used to estimate the second difference of the squared Wasserstein distance.

\subsection{Propagation of the density-ratio bounds}

\begin{lemma}\label{lem:warmness-propagation}
Suppose that $\rho_0^X$ satisfies~\cref{asmp:init_regularity}. Define $g_t=\rho_t^{XY}/\nu^{XY}$ and $q_t=\rho_t^X/\nu^X$. Then for every $t\geq0$ and $(x,y)\in\R^{2d}$,
\[
    \zeta\leq g_t(x,y)\leq\xi\,,\qquad \zeta\leq q_t(x)\leq\xi\,.
\]
\end{lemma}

\begin{proof}
Because $\Psi_t$ preserves volume, by the change-of-variables formula, we have $\rho_t^{XY} = \rho_0^{XY} \circ \Psi_{-t}$ and $\nu^{XY} = \nu^{XY} \circ \Psi_{-t}$.
Here we have also used that $\nu^{XY}$ is invariant along the flow. 
Taking their ratio and using $\rho_0^{XY}=\rho_0^X\otimes\gamma_{d}$ gives, for all $(x,y) \in \R^{2d}$,
\[
    g_t(x,y)=g_0(\Psi_{-t}(x,y))=q_0(\Psi_{-t}^X(x,y))\,.
\]
The upper and lower bounds on $q_0$ therefore give the same bounds on $g_t$. 
We integrate the identity $\rho_t^{XY}(x,y)=g_t(x,y)\nu^X(x)\gamma_{d}(y)$ over $y$ and divide by $\nu^X(x)$, to obtain
\[
    q_t(x)=\int_{\R^d}g_t(x,y)\gamma_{d}(y)\dy\,.
\]
Since $\gamma_{d}$ is a probability density, and using the upper and lower bounds on $g_t$, we obtain from the integral above:
\[
    \zeta\leq q_t(x)\leq\xi\,.
\]
\end{proof}

\subsection{Regularity of the position marginal and its characteristic flow}

\begin{lemma}\label{lem:secondary_regularity}
Suppose that $\rho_0^X$ satisfies~\cref{asmp:init_regularity}, and fix $T>0$. Then the following hold.
\begin{reglist}
    \item\label{reg:u-C1} The position density $\rho_t^X(x)$ and the conditional moments $u_t(x)$, $M_t(x)$, and $\Sigma_t(x)$ are continuously differentiable in $(t,x)$ on $[0,T]\times\R^d$.
    \item\label{reg:integrability} The fields $u_t$, $\Sigma_t$, and $\nabla_x u_t$ are bounded on $[0,T]\times\R^d$. The acceleration satisfies $\|a_t(x)\|\leq C_T(1+\|x\|)$ for some $C_T<\infty$. For each $t\in[0,T]$, the fields $u_t$, $\Sigma_t$, $a_t$, and $v_t$ are square-integrable under $\rho_t^X$.
    \item\label{reg:continuity-eqn} The characteristic equation $\dot\Phi_t(x)=u_t(\Phi_t(x))$ with $\Phi_0(x)=x$ has a unique solution on $[0,T]$ and satisfies $(\Phi_t)_\#\rho_0^X=\rho_t^X$. For each $t\in(0,T)$, the following limits hold in $L^2(\rho_0^X)$ as $h\to0$:
    \begin{align*}
        \frac{\Phi_{t+h}-\Phi_t}{h} &\longrightarrow u_t\circ\Phi_t\,,\\
        \frac{\Phi_{t+h}+\Phi_{t-h}-2\Phi_t}{h^2} &\longrightarrow a_t\circ\Phi_t\,.
    \end{align*}
\end{reglist}
\end{lemma}

\begin{proof}
In the proof, $C_T$ denotes a finite constant that may change from line to line and is independent of $t\in[0,T]$ and of $x,y$.

\paragraph{Continuous differentiability and integrability.}
We express the conditional moments in terms of $g_t$. Define
\[
    j_t(x)=\int_{\R^d}y\,g_t(x,y)\gamma_{d}(y)\dy\,,\qquad
    K_t(x)=\int_{\R^d}yy^\top g_t(x,y)\gamma_{d}(y)\dy\,.
\]
By the definitions of the conditional moments, $u_t=j_t/q_t$ and $M_t=K_t/q_t$. We also have $\rho_t^X=q_t\nu^X$. The bound $g_t\leq\xi$ and the finiteness of the Gaussian moments bound $j_t$ and $K_t$ uniformly in $(t,x)$. Since $q_t\geq\zeta$, both $u_t$ and $M_t$ are bounded, and so is $\Sigma_t=M_t-u_tu_t^\top$.

The derivatives of $g_t(x,y)=q_0(\Psi_{-t}^X(x,y))$ are continuous. By the chain rule, the spatial gradient is bounded by $\|\nabla q_0\|_\infty\|\nabla\Psi_{-t}(x,y)\|_{\op}$. The time derivative is $-\langle\nabla q_0(\Psi_{-t}^X(x,y)),\Psi_{-t}^Y(x,y)\rangle$. The flow estimate in~\Cref{lem:biLip_flow_map} and boundedness of $\Psi_s(0)$ on $[-T,T]$ give
\[
    \|\nabla_x g_t(x,y)\|\leq C_T\,,\qquad
    |\partial_t g_t(x,y)|\leq C_T(1+\|x\|+\|y\|)\,.
\]
For $x$ in any compact set, the right-hand sides are bounded by a constant times $1+\|y\|$. After multiplying by $1$, $\|y\|$, or $\|y\|^2$, the resulting bounds are still integrable against $\gamma_{d}(y)\dy$. We may therefore differentiate the integrals defining $q_t,j_t,K_t$, and dominated convergence gives continuity of their derivatives. It follows that $\rho_t^X,u_t,M_t$ are $C^1$, since $\nu^X$ is $C^1$ and $q_t$ is bounded away from zero.

The same derivative bounds show that $\nabla_x q_t$ and $\nabla_x j_t$ are bounded, and that $|\partial_t q_t(x)|+\|\partial_t j_t(x)\|\leq C_T(1+\|x\|)$. Applying the quotient rule to $u_t=j_t/q_t$, we obtain
\[
    \|\nabla_x u_t(x)\|_{\op}\leq C_T\,,\qquad
    \|\partial_t u_t(x)\|\leq C_T(1+\|x\|)\,.
\]
Combining these two estimates with boundedness of $u_t$ in the identity $a_t=\partial_t u_t+\nabla_x u_t\,u_t$ gives
\begin{equation}\label{eq:acceleration-growth}
    \|a_t(x)\|\leq C_T(1+\|x\|)\,.
\end{equation} 
Thus, $u_t$ and $\Sigma_t$ are square-integrable under $\rho_t^X$ by boundedness, and $a_t$ is square-integrable because $\rho_t^X$ has finite second moment. 
Finally, the optimality of the transport map $R_t$ gives $\int_{\R^d}\|v_t(x)\|^2\rho_t^X(x)\dx=\Wass_2^2(\rho_t^X,\nu^X)<\infty$.

\paragraph{The characteristic flow and its time derivatives.}
We now use the bounds on $u_t$ and $a_t$ to study the characteristic flow for the continuity equation in~\Cref{lem:continuity-momentum-equation}. Since $u_t$ and $\nabla_x u_t$ are bounded on $[0,T]\times\R^d$, the characteristic equation $\dot\Phi_t(x)=u_t(\Phi_t(x))$, $\Phi_0(x)=x$, has a unique solution on $[0,T]$. The method of characteristics gives $(\Phi_t)_\#\rho_0^X=\rho_t^X$; see~\cite[Section~8.1]{ambrosio2005gradient}. The bound on $u_t$ yields $\|\Phi_r(x)-x\|\leq C_T$ for $r\in[0,T]$. Moreover, $\ddot\Phi_r(x)=a_r(\Phi_r(x))$, so the growth bound on $a_r$ gives $\|\ddot\Phi_r(x)\|\leq C_T(1+\|x\|)$.

Fix $t\in(0,T)$. For $h\neq0$ small enough that $t+h\in[0,T]$, the fundamental theorem of calculus gives
\[
    \frac{\Phi_{t+h}(x)-\Phi_t(x)}{h}
    =\int_0^1 u_{t+rh}(\Phi_{t+rh}(x))\,\mathrm{d}r\,.
\]
By continuity, this quotient converges pointwise to $u_t(\Phi_t(x))$. Both the quotient and its limit are uniformly bounded in $x$, so dominated convergence applied to the squared error gives convergence in $L^2(\rho_0^X)$.

For $h>0$ with $t\pm h\in[0,T]$, integrating $\ddot\Phi$ twice gives
\[
    \frac{\Phi_{t+h}(x)+\Phi_{t-h}(x)-2\Phi_t(x)}{h^2}
    =\int_{-1}^1(1-|r|)\,a_{t+rh}(\Phi_{t+rh}(x))\,\mathrm{d}r\,.
\]
Since the weight has integral one, the integral is an average of $a_{t+rh}(\Phi_{t+rh}(x))$ over $[t-h,t+h]$. 
By continuity, it converges pointwise to $a_t(\Phi_t(x))$. The quotient and its limit have norm at most $C_T(1+\|x\|)$, hence the squared error is dominated by $C_T(1+\|x\|^2)$, which is integrable under $\rho_0^X$. 
This proves convergence in $L^2(\rho_0^X)$. 
The centered quotient is an even function of $h$, so the limit holds as $h\to0$.
\end{proof}

%%%%%%%%%%%%%%%
\section{Bound on average \textsf{KL} divergence under regular initialization}
\label[appendix]{apdx:Wass22_proof}

In this appendix, we prove the integrated Wasserstein and \textsf{KL} inequality in~\cref{Lem:Wass-KL-integratedKey} under the regularity assumption on the initial distribution in~\cref{asmp:init_regularity}. We first derive a formula for the acceleration field $a_t$ (\cref{lem:at_representation}) and use the characteristic flow of the position marginal to bound the centered second difference of $\tfrac12\Wass_2^2(\rho_t^X,\nu^X)$ (\cref{lem:Wass22-FD}). We then bound the resulting integral in terms of the \textsf{KL} divergences at times zero and $t$, with an additional Wasserstein term when $M>0$ (\cref{Lem:KLBoundFormula}). Finally, we integrate the bound by a concavity argument that does not assume the existence of an ordinary second time derivative (\cref{Lem:AvgKLRegularity}).

\subsection{Formula for the acceleration field}
\label[appendix]{Sec:at_representationProof}

\begin{lemma}\label{lem:at_representation}
Assume the setting of~\cref{Lem:Wass-KL-integratedKey} and~\cref{asmp:init_regularity}. Then for every $t\geq0$ and $x\in\R^d$,
\begin{equation*}
    a_t(x)=-\nabla f(x)-\frac{1}{\rho_t^X(x)}\nabla\cdot(\rho_t^X\Sigma_t)(x)\,.
\end{equation*}

\end{lemma}

\begin{proof}
By~\cref{reg:u-C1}, all terms in the continuity and momentum equations are continuous, so the distributional identities hold pointwise for $t>0$ and extend to $t=0$ by continuity. 
We recall the momentum equation from~\cref{lem:continuity-momentum-equation}: $\partial_t(\rho_t^X u_t)+\nabla\cdot(\rho_t^X M_t)+\rho_t^X\nabla f=0$.
We substitute $M_t=u_tu_t^\top+\Sigma_t$ and use the product rule to obtain
\begin{align*}
\partial_t(\rho_t^X u_t)+\nabla\cdot(\rho_t^X u_tu_t^\top)
&=\rho_t^X\bigl(\partial_t u_t+(\nabla u_t)u_t\bigr)
+ u_t\bigl(\partial_t\rho_t^X+\nabla\cdot(\rho_t^X u_t)\bigr)\\
&=\rho_t^X a_t\,,
\end{align*}
where the second term vanishes by the continuity equation. 
Thus, the momentum equation becomes
\begin{equation*}
    \rho_t^X a_t+\nabla\cdot(\rho_t^X\Sigma_t)+\rho_t^X\nabla f=0\,.
\end{equation*} 
Since $\rho_t^X(x)\geq\zeta\nu^X(x)>0$ by~\cref{lem:warmness-propagation}, we can divide by $\rho_t^X(x)$ to obtain the formula.
\end{proof}

\subsection{An upper bound on the centered second difference of \texorpdfstring{$\Wass_2^2$}{W2 squared}}
\label[appendix]{Sec:Wass22-FDProof}

The following estimate is related to the Wasserstein derivative bounds of~\cite[Theorems~1 and~2]{Carrillo_2018}. Here we use the $L^2$ limits in~\cref{reg:continuity-eqn} to obtain a bound at every interior time, without assuming that the squared Wasserstein distance is twice differentiable.

\begin{lemma}\label{lem:Wass22-FD}
Assume the setting of~\cref{Lem:Wass-KL-integratedKey} and~\cref{asmp:init_regularity}, and fix $T>0$. Let $F(t)=\frac12\Wass_2^2(\rho_t^X,\nu^X)$. Then $F$ is continuous on $[0,T]$ and has right derivative $F'_+(0)=0$. For each $t\in(0,T)$,
\[
    \limsup_{h\downarrow0}\frac{F(t+h)+F(t-h)-2F(t)}{h^2}
    \leq \int_{\R^d}\bigl(\|u_t(x)\|^2+\langle v_t(x),a_t(x)\rangle\bigr)\rho_t^X(x)\dx\,.
\]

\end{lemma}

\begin{proof}
\textit{Continuity and the initial derivative.}
By~\cref{reg:continuity-eqn}, $(\Phi_t)_\#\rho_0^X=\rho_t^X$. Coupling $\Phi_s$ and $\Phi_t$ at the same initial position gives
\[
    \Wass_2(\rho_s^X,\rho_t^X)
    \leq\|\Phi_s-\Phi_t\|_{L^2(\rho_0^X)}\leq C_T|s-t|\,,
\]
where the last inequality uses boundedness of $u_t$ on $[0,T]\times\R^d$. 
By the triangle inequality, this implies the continuity of $F$. 
To compute $F'_+(0)$, recall from~\cref{reg:continuity-eqn} that $\Phi_0(x)=x$ and $\dot\Phi_s(x)=u_s(\Phi_s(x))$. Differentiating this equation in time and using~\eqref{Eq:DefAccField} gives
\[
    \ddot\Phi_s(x)=\bigl(\partial_s u_s+(\nabla u_s)u_s\bigr)(\Phi_s(x))
    =a_s(\Phi_s(x))\,.
\]
Since $Y_0\sim\N(0,\I_d)$ is independent of $X_0$, we have $u_0=0$, so $\dot\Phi_0(x)=0$. Integrating the acceleration twice from time zero therefore gives
\[
    \Phi_h(x)-x=\int_0^h(h-s)a_s(\Phi_s(x))\ds\,.
\]
The estimate~\eqref{eq:acceleration-growth} and boundedness of $u_t$ give $\|a_s(\Phi_s(x))\|\leq C_T(1+\|x\|)$ for $s\in[0,T]$. Since $\rho_0^X$ has finite second moment, we obtain $\Wass_2(\rho_h^X,\rho_0^X)\leq\|\Phi_h-\operatorname{Id}\|_{L^2(\rho_0^X)}=O(h^2)$.
Then by the triangle inequality, 
$$|F(h)-F(0)| \le \frac12 \left[\Wass_2(\rho_h^X,\nu^X)+\Wass_2(\rho_0^X,\nu^X) \right] \, \Wass_2(\rho_h^X,\rho_0^X) = O(h^2) \,.$$ Therefore, $F'_+(0)=0$.

\noindent\textit{The centered second difference.}
Fix $t\in(0,T)$ and take $h>0$ small enough that $t\pm h\in[0,T]$. The map $R_t\circ\Phi_t$ pushes $\rho_0^X$ forward to $\nu^X$. 
We keep this destination fixed, and couple it with each of $\Phi_{t+h}$ and $\Phi_{t-h}$ under $\rho_0^X$. 
These (suboptimal) couplings bound the costs for $F(t+h)$ and $F(t-h)$ from above, whereas the corresponding coupling at $t$ attains $F(t)$ by construction. 
Let $\Delta_\pm=\Phi_{t\pm h}-\Phi_t$, and define $V=v_t\circ\Phi_t=\Phi_t-R_t\circ\Phi_t$. 
Expanding the squared costs and subtracting the cost at $t$ gives
\begin{align*}
\frac{F(t+h)+F(t-h)-2F(t)}{h^2}
&\leq \frac12\left\|\frac{\Delta_+}{h}\right\|_{L^2(\rho_0^X)}^2
 + \frac12\left\|\frac{\Delta_-}{h}\right\|_{L^2(\rho_0^X)}^2
 + \left\langle V,\frac{\Delta_++\Delta_-}{h^2}\right\rangle_{L^2(\rho_0^X)}\,.
\end{align*}
By~\cref{reg:continuity-eqn}, the first two quotients converge in $L^2(\rho_0^X)$ to $u_t\circ\Phi_t$ and $-u_t\circ\Phi_t$, respectively, and the centered quotient converges to $a_t\circ\Phi_t$. Since $V\in L^2(\rho_0^X)$ by~\cref{reg:integrability} and the pushforward identity, the right-hand side converges to $\|u_t\circ\Phi_t\|_{L^2(\rho_0^X)}^2+\langle v_t\circ\Phi_t,a_t\circ\Phi_t\rangle_{L^2(\rho_0^X)}$. Changing variables by $(\Phi_t)_\#\rho_0^X=\rho_t^X$ gives the desired bound.
\end{proof}

\subsection{Bounding the centered second difference by \textsf{KL} divergence}
\label{sub_apdx:proof_KL_from_mess}

We now bound the integral on the right-hand side of~\cref{lem:Wass22-FD}.

\begin{lemma}\label{Lem:KLBoundFormula}
Assume~\Cref{asmp:init_regularity} and the setting
of~\Cref{Lem:Wass-KL-integratedKey}.
In particular, assume that \(\nu^X\propto\exp(-f)\) is
\(M\)-semi-log-concave.
Then, for every \(T>0\) and \(t\in(0,T)\),
\begin{align*}
&\int_{\R^d}
\bigl(\|u_t(x)\|^2+\langle v_t(x),a_t(x)\rangle\bigr)
\rho_t^X(x)\dx\leq
2\KL(\rho_0^X\dvert\nu^X)-3\KL(\rho_t^X\dvert\nu^X)
+\frac M2\Wass_2^2(\rho_t^X,\nu^X) \,.
\end{align*}

\end{lemma}

\begin{proof}
Fix $T>0$ and $t\in(0,T)$. We first integrate by parts with a cutoff, using the distributional Hessian of the Brenier potential, following the approach in~\cite{lu2026sharphypocoerciveentropydecay}. 
We then bound the velocity and covariance terms together by choosing a Gaussian whose precision matrix is the density of the absolutely continuous part of this Hessian.

\paragraph{Integration by parts.}
Let $R_t=\nabla\varphi_t$, where $\varphi_t$ is the convex Brenier potential for the optimal transport map from $\rho_t^X$ to $\nu^X$. 
We recall from~\cite[Theorem~6.8]{evans2015measure} that the distributional derivative of $\nabla\varphi_t$, denoted by $[D^2\varphi_t]$, is a positive semidefinite matrix-valued Radon measure.
We write its Lebesgue decomposition as
\[
    [D^2\varphi_t](\dx)=G_t(x)\dx+N_t(x) \,\sigma_t(\dx) \,,
\]
where $\sigma_t$ is nonnegative and singular with respect to Lebesgue measure, and both $G_t$ and $N_t$ are positive semidefinite almost everywhere with respect to their corresponding measures. 
The density $G_t$ of the absolutely continuous part agrees Lebesgue-a.e.\ with the Alexandrov Hessian of $\varphi_t$.
The Monge--Amp\`ere identity~\citep[Theorem~11.1]{villani2009optimal} gives
\[
    \rho_t^X(x)=\nu^X(R_t(x))\det G_t(x)
    \qquad\text{for $\rho_t^X$-a.e.\ $x$} \,.
\]
Since the densities $\rho_t^X(x)$ and $\nu^X(R_t(x))$ are positive, we have $\det G_t(x)>0$, so the matrix $G_t(x)$ is positive definite for $\rho_t^X$-a.e.\ $x$.

Let $\chi \colon \R^d\to[0,1]$ be a smooth cutoff function that takes value one on the unit ball and vanishes outside the ball of radius two. 
For $R > 0$, let $\chi_R(x) \deq \chi(x/R)$, so $\|\nabla\chi_R\|_\infty\leq C/R$ for some constant $C \in (0,\infty)$. 
Define
\[
    I_R \deq \int_{\R^d}\chi_R(x)
    \bigl(\|u_t(x)\|^2+\langle v_t(x),a_t(x)\rangle\bigr) \, 
    \rho_t^X(x)\dx\,.
\]
Substituting the formula for $a_t$ from~\Cref{lem:at_representation} into $I_R$ gives
\begin{align*}
    I_R&=\int_{\R^d}\chi_R\bigl(\|u_t\|^2-\langle v_t,\nabla f\rangle\bigr)\rho_t^X\dx -\int_{\R^d}\chi_R\langle v_t,\nabla\cdot(\rho_t^X\Sigma_t)\rangle\dx\,.
\end{align*}
In the second integral, the factor $1/\rho_t^X$ in $a_t$ cancels the factor $\rho_t^X$ in the definition of $I_R$. We integrate this term by parts to move the derivative onto $v_t$. First, the product rule gives
\[
    -\int_{\R^d} \chi_R\langle v_t,\nabla\cdot(\rho_t^X\Sigma_t)\rangle\dx
    =-\int_{\R^d} \langle v_t,\nabla\cdot(\chi_R\rho_t^X\Sigma_t)\rangle\dx
    +\varepsilon_R \, ,
\]
where 
$$\varepsilon_R \deq \int_{\R^d}\langle v_t,\Sigma_t\nabla\chi_R\rangle\rho_t^X\dx \,.$$
By~\Cref{reg:u-C1}, the matrix field $\chi_R\rho_t^X\Sigma_t$ is continuously differentiable and compactly supported. Since $v_t(x)=x-\nabla\varphi_t(x)$, its distributional derivative is the derivative of the identity map minus the Hessian measure of $\varphi_t$; that is,
\[
    Dv_t(\dx)=(\I_d-G_t(x))\dx-N_t(x)\,\sigma_t(\dx) \, .
\]
In particular, $v_t$ is locally of bounded variation. The integration-by-parts identity for smooth test fields extends to $\chi_R\rho_t^X\Sigma_t$ by approximation in $C^1$ on a fixed compact set, since $Dv_t$ is locally finite. Applying this identity and substituting the two parts of $Dv_t$ gives
\begin{align*}
    -\int_{\R^d} \langle v_t,\nabla\cdot(\chi_R\rho_t^X\Sigma_t)\rangle\dx
    &=\int_{\R^d} \langle \chi_R\rho_t^X\Sigma_t,Dv_t(\dx)\rangle_{\mathsf F} \\
    &=\int_{\R^d} \chi_R\rho_t^X\langle \I_d-G_t,\Sigma_t\rangle_{\mathsf F}\dx
      -\int_{\R^d} \chi_R\rho_t^X\langle N_t,\Sigma_t\rangle_{\mathsf F}\,\mathrm{d}\sigma_t \, .
\end{align*}
Since $N_t$ and $\Sigma_t$ are positive semidefinite, we have
$\langle N_t,\Sigma_t\rangle_{\mathsf F}=\Tr(\Sigma_t^{1/2}N_t\Sigma_t^{1/2})\geq 0$.
The integral against $\sigma_t$ is nonnegative and appears with a minus sign, so we may drop it to get an upper bound. 
Adding back the squared velocity and force terms in $I_R$, we obtain
\[
    I_R \leq \int_{\R^d} \chi_R \, B_t \, \rho_t^X \dx + \varepsilon_R \, ,
\]
where
\[
    B_t(x) \deq \|u_t(x)\|^2+\langle\I_d-G_t(x),\Sigma_t(x)\rangle_{\mathsf F}
    -\langle v_t(x),\nabla f(x)\rangle \, .
\]
At the fixed time $t$, both $v_t$ and $\Sigma_t$ are square-integrable under $\rho_t^X$ by~\cref{reg:integrability}. Using $\|\nabla\chi_R\|_\infty\leq C/R$ and Cauchy--Schwarz, we obtain
\[
    |\varepsilon_R|\leq \frac{C}{R}
    \left(\int_{\R^d}\|v_t\|^2\rho_t^X\dx\right)^{1/2}
    \left(\int_{\R^d}\|\Sigma_t\|_{\mathsf F}^2\rho_t^X\dx\right)^{1/2}
    \longrightarrow 0 \qquad \text{ as } \qquad R \to \infty \,.
\]

\paragraph{Pointwise bound.}
For $\rho_t^X$-a.e.\ $x$, we write $p_x \deq \rho_t^{Y\mid X=x}$ and define $q_x \deq \N(u_t(x),G_t(x)^{-1})$. 
The conditional second moment and $\KL(p_x\dvert\gamma_{d})$ are finite by the moment bound and the chain rule for the joint \textsf{KL} divergence. 
By the formula for the Gaussian density,
\[
    2\log\frac{q_x(y)}{\gamma_{d}(y)}
    =\log\det G_t(x)+\|y\|^2
      -(y-u_t(x))^\top G_t(x)(y-u_t(x))\,.
\]
Since $p_x$ has a finite second moment, the quadratic polynomial
$\log(q_x/\gamma_{d})$ is integrable under $p_x$. Therefore, $\KL(p_x\dvert q_x)
    = \KL(p_x\dvert\gamma_{d})
      - \int_{\R^d} \log(q_x(y)/\gamma_{d}(y))\,p_x(y)\,\mathrm{d}y
    < \infty$.
Taking the expectation of both sides of the formula under $p_x$ gives
\begin{align*}
    2\KL(p_x\dvert\gamma_{d})-2\KL(p_x\dvert q_x)
    &=\|u_t(x)\|^2+\langle\I_d-G_t(x),\Sigma_t(x)\rangle_{\mathsf F}
      +\log\det G_t(x)\,,
\end{align*}
since the expectations of $\|Y\|^2$ and $(Y-u_t(x))^\top G_t(x)(Y-u_t(x))$ under $p_x$ are $\|u_t(x)\|^2+\Tr(\Sigma_t(x))$ and $\Tr(G_t(x)\Sigma_t(x))$, respectively. 
Next, we drop the term $\KL(p_x\dvert q_x)\geq0$, and use the Monge--Amp\`ere identity to write $\log\det G_t(x)=\log\frac{\rho_t^X(x)}{\nu^X(x)}+f(R_t(x))-f(x)$, resulting in
$$\|u_t(x)\|^2+\langle\I_d-G_t(x),\Sigma_t(x)\rangle_{\mathsf F} \le 2\KL(p_x\dvert\gamma_{d}) - \log\frac{\rho_t^X(x)}{\nu^X(x)} - f(R_t(x)) + f(x) \,.$$
By the $M$-semi-convexity of $f$, we also have
$$-\langle v_t(x),\nabla f(x)\rangle\leq f(R_t(x))-f(x)+\frac M2\|v_t(x)\|^2 \,.$$
Combining the two bounds above, the terms $f(R_t(x)) - f(x)$ cancel, and we obtain
\[
    B_t(x)\leq 2\KL(p_x\dvert\gamma_{d})-\log\frac{\rho_t^X(x)}{\nu^X(x)}
       +\frac M2\|v_t(x)\|^2\,.
\]

\paragraph{Removing the cutoff.}
Substituting the pointwise bound for $B_t$ into the estimate for $I_R$, we have
\[
    I_R \leq \int_{\R^d}\chi_R(x)
    \left(2\KL(p_x\dvert\gamma_{d})-\log\frac{\rho_t^X(x)}{\nu^X(x)}
    +\frac M2\|v_t(x)\|^2\right)\rho_t^X(x)\dx+\varepsilon_R \,,
\]
for every $R>0$. The expression in parentheses is integrable against $\rho_t^X$: the conditional $\KL$ is integrable by the chain rule, the log density ratio is bounded by~\Cref{lem:warmness-propagation}, and the squared displacement is integrable by optimality of $R_t$. For the left-hand side, the integrand is bounded in absolute value by $\|u_t\|^2+\|v_t\|\|a_t\|$, which is integrable by the fixed-time square-integrability bounds in~\cref{reg:integrability} and Cauchy--Schwarz.
As $R\to\infty$, we have $0\leq\chi_R\leq1$ and $\chi_R(x)\to1$ for every $x$, so dominated convergence lets us remove $\chi_R$ both in $I_R$ and in the integral on the right. Since $\varepsilon_R\to0$, we obtain
\begin{align*}
 \int_{\R^d}\bigl(\|u_t\|^2+\langle v_t,a_t\rangle\bigr)\rho_t^X\dx
 &\leq 2\int_{\R^d}\KL(p_x\dvert\gamma_{d})\rho_t^X(x)\dx
    -\KL(\rho_t^X\dvert\nu^X)+\frac M2\Wass_2^2(\rho_t^X,\nu^X)\\
 &=2\KL(\rho_t^{XY}\dvert\nu^{XY})-3\KL(\rho_t^X\dvert\nu^X)
    +\frac M2\Wass_2^2(\rho_t^X,\nu^X)\\
 &=2\KL(\rho_0^X\dvert\nu^X)-3\KL(\rho_t^X\dvert\nu^X)
    +\frac M2\Wass_2^2(\rho_t^X,\nu^X)\,.
\end{align*}
The first equality above uses the chain rule of $\KL$ divergence.
The second equality uses the relation $\KL(\rho_t^{XY}\dvert\nu^{XY})
    =\KL(\rho_0^{XY}\dvert\nu^{XY})
    =\KL(\rho_0^X\dvert\nu^X)$,
which follows from the conservation of joint $\KL$ divergence (\Cref{Lem:HamFlowConsvJointKL}) and the choice $\rho_0^{XY}=\rho_0^X\otimes\gamma_{d}$ and $\nu^{XY}=\nu^X\otimes\gamma_{d}$.
\end{proof}

\subsection{Integrating the bound under regular initialization}
\label[appendix]{Sec:AvgKLRegularity}

We will use the following criterion for concavity, which does not require differentiability.

\begin{proposition}[{\cite[Theorem 1.4.8]{NiculescuPersson2018}}]
\label{prop:1d-concavity}
Let \(g : (a, b) \to \R\) be a continuous function.
The function \(g\) is concave if for all \(t \in (a, b)\)
\begin{equation*}
    \liminf_{h \to 0} \frac{g(t + h) - 2g(t) + g(t - h)}{h^{2}} \leq 0 \,.
\end{equation*}
\end{proposition}

\begin{lemma}\label{Lem:AvgKLRegularity}
Assume the setting of~\cref{Lem:Wass-KL-integratedKey} and~\cref{asmp:init_regularity}. Then, for every $T>0$,
\begin{multline*}
    \frac12\Wass_2^2(\rho_T^X,\nu^X)+3\int_0^T(T-t)\KL(\rho_t^X\dvert\nu^X)\dt
    -\frac M2\int_0^T(T-t)\Wass_2^2(\rho_t^X,\nu^X)\dt\\
    \leq\frac12\Wass_2^2(\rho_0^X,\nu^X)+T^2\KL(\rho_0^X\dvert\nu^X)\,.
\end{multline*}

\end{lemma}

\begin{proof}
Fix $T>0$. Combining~\cref{lem:Wass22-FD,Lem:KLBoundFormula}, we obtain, for each $t\in(0,T)$,
\begin{multline}\label{eq:Wass22-FD+KLBoundFormula}
    \limsup_{h\to0}\frac{\tfrac12\Wass_2^2(\rho_{t+h}^X,\nu^X)+\tfrac12\Wass_2^2(\rho_{t-h}^X,\nu^X)-\Wass_2^2(\rho_t^X,\nu^X)}{h^2}\\
    \leq 2\KL(\rho_0^X\dvert\nu^X)-3\KL(\rho_t^X\dvert\nu^X)+\frac M2\Wass_2^2(\rho_t^X,\nu^X)\,.
\end{multline}

For $t\in[0,T]$, let $F(t)=\tfrac12\Wass_2^2(\rho_t^X,\nu^X)$ and denote the right-hand side of~\eqref{eq:Wass22-FD+KLBoundFormula} by $J(t)$. To integrate this estimate without assuming that $F$ is twice differentiable, define
\[
    G(t)=\int_0^t(t-s)J(s)\ds\,,\qquad 0\leq t\leq T\,.
\]
The continuity of the squared Wasserstein distance and the \textsf{KL} divergence along the flow (\Cref{lem:Wass22-FD,Lem:KLContinuous}) gives continuity of $J$. Thus, $G$ is twice continuously differentiable with $G''(t)=J(t)$ and $G(0)=G'(0)=0$.

The centered second difference of $G$, divided by $h^2$, converges to $J(t)$. Subtracting it from the bound for $F$ gives
\[
    \limsup_{h\to0}\frac{(F-G)(t+h)+(F-G)(t-h)-2(F-G)(t)}{h^2}\leq0\,.
\]
The concavity criterion in~\cref{prop:1d-concavity} applies, since the lower limit is at most the upper limit. Hence $F-G$ is concave on $(0,T)$ and, by continuity, on $[0,T]$. Its right derivative at zero is $F'_+(0)-G'(0)=0$ by~\cref{lem:Wass22-FD}. Concavity bounds the secant slope from $0$ to $T$ by this right derivative, so $(F-G)(T)\leq(F-G)(0)=F(0)$. Substituting the definition of $G(T)$ and integrating the constant term $2\KL(\rho_0^X\dvert\nu^X)$ against $T-t$ gives the claimed inequality.
\end{proof}

%%%%%%%%%%%%%%%%%%%
\section{Approximation argument for regularity of initial distribution}
\label[appendix]{apdx:approximation-proof}

%%%%%%%
\subsection{\texorpdfstring{Proof of~\Cref{Lem:Wass-KL-integratedKey} for general initialization}{Proof of the integrated inequality for general initialization}}
\label[appendix]{Sec:Wass-KL-integratedKeyProof}

In~\Cref{Sec:AvgKLRegularity}, we proved~\Cref{Lem:Wass-KL-integratedKey} under the regularity assumption on $\rho_0^X$. 
We now remove this assumption by approximating the initial distribution in Wasserstein distance and in $\KL$ divergence.
Our construction adapts the density-ratio approximation of~\cite[Lemma 7.3]{lu2026sharphypocoerciveentropydecay}.

\paragraph{Regular approximation.}
Assume $\nu^X \in\PtwoacfsRd$ is $L$-log-smooth for some $L \in (0,\infty)$.
Let $\rho_0^X\in\PtwoacfsRd$ with $\KL(\rho_0^X\dvert\nu^X)<\infty$. 
We show in~\cref{lem:ConsRegAppr} below that there exists a sequence $(\rho_{0,n}^X)_{n \ge 1}$ satisfying~\Cref{asmp:init_regularity} such that
\[
    \Wass_2(\rho_{0,n}^X,\rho_0^X)\longrightarrow 0 \,,
    \qquad\text{ and }\qquad
    \KL(\rho_{0,n}^X\dvert\nu^X)\longrightarrow\KL(\rho_0^X\dvert\nu^X) \ .
\]
Using this approximating sequence, we can prove~\Cref{Lem:Wass-KL-integratedKey}.

\begin{proof}[Proof of~\Cref{Lem:Wass-KL-integratedKey}]
Fix $T>0$ (the case $T = 0$ is trivial). 
For each $n \ge 1$, take a $\Wass_2$-optimal coupling $(X_{0,n},X_0)$ of $\rho_{0,n}^X$ and $\rho_0^X$.
Let $Y_0 \sim \gamma_{d}$ be independent. 
We evolve both initial conditions $(X_{0,n},Y_0)$ and $(X_0,Y_0)$ along the Hamiltonian flow with the same initial velocity $Y_0$, and denote their position laws at time $t$ by $X_{t,n} \sim \rho_{t,n}^X$ and $X_t \sim \rho_t^X$, respectively.
By~\Cref{lem:biLip_flow_map},
\[
    \sup_{0\leq t\leq T}\Wass_2(\rho_{t,n}^X,\rho_t^X)
    \,\leq\, e^{(1+L)T} \, \Wass_2(\rho_{0,n}^X,\rho_0^X)
    \longrightarrow 0 \ .
\]
Let $\delta_n=\sup_{0\leq t\leq T}\Wass_2(\rho_{t,n}^X,\rho_t^X)$ and $B=\sup_{0\leq t\leq T}\Wass_2(\rho_t^X,\nu^X)$. The estimate above gives $\delta_n\to0$, and $B<\infty$ by~\Cref{lem:second-moment-finiteness}. For every $t\in[0,T]$, the triangle inequality gives
\begin{align*}
    \bigl|\Wass_2(\rho_{t,n}^X,\nu^X)-\Wass_2(\rho_t^X,\nu^X)\bigr|
    &\leq \Wass_2(\rho_{t,n}^X,\rho_t^X) \leq \delta_n \,, \\
    \Wass_2(\rho_{t,n}^X,\nu^X)+\Wass_2(\rho_t^X,\nu^X)
    &\leq 2\Wass_2(\rho_t^X,\nu^X)+\Wass_2(\rho_{t,n}^X,\rho_t^X) \leq 2B+\delta_n \ .
\end{align*}
Multiplying these two bounds gives the uniform convergence of the squared Wasserstein distances on $[0,T]$:
\[
    \sup_{0\leq t\leq T}
    \bigl|\Wass_2^2(\rho_{t,n}^X,\nu^X)-\Wass_2^2(\rho_t^X,\nu^X)\bigr|
    \leq \delta_n(2B+\delta_n)\longrightarrow0 \ .
\]

We apply~\Cref{Lem:AvgKLRegularity} to each $\rho_{0,n}^X$. 
By the uniform convergence above, the two Wasserstein terms on the left-hand side of~\eqref{Eq:KLIntegral} satisfy, as $n \to \infty$,
\begin{align*}
    \frac12\Wass_2^2(\rho_{T,n}^X,\nu^X)
    &\longrightarrow \frac12\Wass_2^2(\rho_T^X,\nu^X) \,, \\
    -\frac{M}{2}\int_0^T (T-t) \, \Wass_2^2(\rho_{t,n}^X,\nu^X)\dt
    &\longrightarrow -\frac{M}{2}\int_0^T (T-t)\,\Wass_2^2(\rho_t^X,\nu^X)\dt \ .
\end{align*}
For each $t\in[0,T]$, the Wasserstein convergence above implies that $\rho_{t,n}^X$ converges weakly to $\rho_t^X$.
The \textsf{KL} integrands are continuous in time by~\Cref{Lem:KLContinuous}. 
Then the lower semicontinuity of \textsf{KL} divergence and Fatou's lemma give
\[
    \int_0^T (T-t)\,\KL(\rho_t^X\dvert\nu^X)\dt
    \leq \liminf_{n\to\infty}
    \int_0^T (T-t)\,\KL(\rho_{t,n}^X\dvert\nu^X)\dt \ .
\]
The two terms on the right-hand side of~\eqref{Eq:KLIntegral} converge by the regular approximation. Taking the lower limit in the inequality for $\rho_{0,n}^X$ proves~\eqref{Eq:KLIntegral} for $\rho_0^X$.
\end{proof}

%%%%%%%%%%%%%%%%%%%%%%
\subsection{Construction of the regular approximation}\label[appendix]{Sec:ConsRegAppr}

\begin{lemma}
\label{lem:ConsRegAppr}
For any $\rho_0^X \in \PtwoacfsRd$ such that
$\KL(\rho_0^X \dvert \nu^X)<\infty$, there exists a
sequence $\{\rho_{0,n}^X\}_{n\geq 1}$ satisfying
\Cref{asmp:init_regularity} such that
\[
    \lim_{n \to \infty} \Wass_2(\rho_{0,n}^X,\rho_0^X) = 0 \,,
    \qquad \text{and} \qquad
    \lim_{n \to \infty} \KL(\rho_{0,n}^X\dvert\nu^X)
    = \KL(\rho_0^X\dvert\nu^X) \ .
\]
\end{lemma}

\begin{proof}
We first construct a regular approximation of $q_0 = \rho_0^X / \nu^X$,
and then normalise it to obtain a probability density.

\paragraph{Step 1: Approximation of the density ratio.}
For each $n\geq 1$, define
\[
    \bar q_n(x)
    \deq \min\{q_0(x),n\}
    \mathbf{1}_{\{\|x\|\leq n\}}+\frac{1}{n} \ .
\]
Since $\bar q_n\to q_0$ almost everywhere and
$\bar q_n\leq q_0+1$, by dominated convergence,
\begin{align*}
    \lim_{n \to \infty} \int_{\R^d}(1+\|x\|^2)
    |\bar q_n(x)-q_0(x)|\,\nu^{X}(x) \dx
    &= 0 \,,\\
    \lim_{n \to \infty} \int_{\R^d}\bar q_n(x)\log\bar q_n(x)\,\nu^{X}(x) \dx
    &= \KL(\rho_0^X\dvert\nu^X) \ .
\end{align*}
For the first limit, we use the fact that $\rho_0^X$
and $\nu^X$ have finite second moments.
For the second limit, we use the bound
\[
    |\bar q_n\log\bar q_n|
    \leq (q_0+1)\log(q_0+1)+1/e
    \leq q_0\log q_0+(q_0+1)\log 2+1/e \ .
\]
The first inequality above uses $\bar q_n\leq q_0+1$ and $s\log s\geq-1/e$, whereas the second follows by applying convexity of $s\mapsto s\log s$ at the midpoint of $q_0$ and $1$.
The right-hand side is integrable with respect to $\nu^X$ since $\KL(\rho_0^X\dvert\nu^X)<\infty$ and $\int_{\R^d}q_0(x)\nu^X(x)\dx=1$.
We now smooth $\bar q_n$ while preserving its upper
and lower bounds.
Let $Z\sim\gamma_{d}$ and, for each
$\varepsilon>0$, define
\[
    r_{n,\varepsilon}(x)
    \deq \E_{Z \sim \gamma_{d}}[\bar q_n(x+\varepsilon\cdot Z)] \ .
\]
Since $1/n\leq\bar q_n\leq n+1/n$, we immediately have for every \(x \in \R^{d}\) that \(1/n \leq r_{n, \varepsilon}(x) \leq n + 1/n\).
Moreover, since $\bar q_n-1/n$ is bounded and compactly
supported, by the $L^1$ convergence of Gaussian
convolutions \citep[Theorem 8.14]{folland1999real}, we have
\[
    \lim_{\varepsilon \to 0}\int_{\R^d}|r_{n,\varepsilon}(x)-\bar q_n(x)|\dx
     = 0~.
\]

Since $\nu^X$ is continuous, for every $R>0$, define
\[
    C_R\deq
    \sup_{\|x\|\leq R}(1+\|x\|^2)\nu^X(x)
    <\infty \ .
\]
Consequently, using $|r_{n,\varepsilon}-\bar q_n|\leq n$,
we obtain
\begin{align*}
    \int_{\R^d}(1&+\|x\|^2)
    |r_{n,\varepsilon}(x)-\bar q_n(x)| \nu^{X}(x) \dx \\
    &= \int_{x : \|x\| \leq R} (1 + \|x\|^{2}) |r_{n,\varepsilon}(x) - \bar{q}_{n}(x)|\nu^{X}(x) \dx \\
    &\quad + \int_{x : \|x\| > R} (1 + \|x\|^{2}) |r_{n,\varepsilon}(x) - \bar{q}_{n}(x)| \nu^{X}(x) \dx \\
    &\leq C_{R} \int_{\R^{d}} |r_{n,\varepsilon}(x) - \bar{q}_{n}(x)| \dx + n \int_{x : \|x\| > R} (1 + \|x\|^{2}) \nu^{X}(x) \dx~.
\end{align*}
Taking $\varepsilon\to 0$ and using the $L^1$
convergence established above, we obtain, for every
$R>0$,
\[
    \limsup_{\varepsilon\to 0}
    \int_{\R^d}(1+\|x\|^2)
    |r_{n,\varepsilon}(x)-\bar q_n(x)|\,\nu^{X}(x)\d x
    \leq n\int_{x : \|x\| > R}(1+\|x\|^2)\,\nu^{X}(x) \d x~.
\]
Since $\nu^X$ has finite second moment, the right-hand
side converges to zero as $R\to\infty$ by dominated
convergence. Hence, by nonnegativity,
\[
    \lim_{\varepsilon\to 0}
    \int_{\R^d}(1+\|x\|^2)
    |r_{n,\varepsilon}(x)-\bar q_n(x)|\,\nu^{X}(x) \d x
    =0 \ .
\]
Consequently, for each $n\geq 1$, we may choose
$\varepsilon_n\in(0,1/n]$ sufficiently small that
$r_n\deq r_{n,\varepsilon_n}$ satisfies
\[
     \int_{\R^d}(1+\|x\|^2)
    |r_n(x)-\bar q_n(x)|\,\nu^{X}(x) \d x
    \leq \frac{1}{n^2} \ .
\]
Moreover, since $s\mapsto s\log s$ is Lipschitz on
$[1/n,n+1/n]$ with Lipschitz constant at most
$1+\log(n+1)$, we have
\[
    \int_{\R^d}|r_n(x)\log r_n(x)-\bar q_n(x)\log\bar q_n(x)|\,
    \nu^{X}(x) \d x \leq \int_{\R^{d}} (1 + \log(n + 1)) |r_{n}(x) - \bar{q}_{n}(x)| \nu^{X}(x) \d x~.
\]
By the triangle inequality, the weighted error between $r_n$ and $q_0$ satisfies
\[
    \int_{\R^d}(1+\|x\|^2)|r_n(x)-q_0(x)|\,\nu^X(x)\dx
    \leq \frac{1}{n^2}+\int_{\R^d}(1+\|x\|^2)|\bar q_n(x)-q_0(x)|\,\nu^X(x)\dx \ .
\]
The right-hand side converges to zero, and the entropy error above is at most $(1+\log(n+1))/n^2$. Therefore, using the two convergence properties of $\bar q_n$, we obtain
\begin{align}
    \lim_{n\to\infty}\int_{\R^d}(1+\|x\|^2)|r_n(x)-q_0(x)|\,\nu^X(x)\dx &=0 \,,
    \label{eq:ConsRegAppr:moment} \\
    \lim_{n\to\infty}\int_{\R^d}r_n(x)\log r_n(x)\,\nu^X(x)\dx &=\KL(\rho_0^X\dvert\nu^X) \ .
    \label{eq:ConsRegAppr:entropy}
\end{align}

\paragraph{Step 2: Normalisation and convergence.}
Define
\[
    c_n\deq \int_{\R^d}r_n(x)\,\nu^X(x) \d x \,,
    \qquad
    \rho_{0,n}^X\deq \frac{1}{c_n}\cdot r_n\nu^X \ .
\]
By construction, $\rho_{0,n}^X$ is a probability density
with finite second moment. Moreover,
\[
    \frac{1}{nc_n}
    \leq \frac{\rho_{0,n}^X}{\nu^X}
    \leq \frac{n+1/n}{c_n} \ .
\]
Since $r_n$ is the Gaussian convolution of a bounded
function, it is continuously differentiable with
\[
    \nabla r_n(x)
    = \frac{1}{\varepsilon_n}\cdot
    \E[\bar q_n(x+\varepsilon_n\cdot Z)\cdot Z] \ .
\]
In particular,
\[
    \sup_{x\in\R^d}\|\nabla r_n(x)\|
    \leq \frac{n+1/n}{\varepsilon_n}\E[\|Z\|]
    <\infty \ .
\]
Hence, $\rho_{0,n}^X/\nu^X=r_n/c_n$ satisfies
\Cref{asmp:init_regularity}.

Since $q_0\nu^X=\rho_0^X$ is a probability density, $q_0$ integrates to one against $\nu^X$. The definition of $c_n$ gives
\[
    |c_n-1|\leq \int_{\R^d}|r_n(x)-q_0(x)|\,\nu^X(x)\dx\longrightarrow0 \,,
\]
where the convergence follows from~\eqref{eq:ConsRegAppr:moment}, since the weight $1+\|x\|^2$ is at least one.

We next show that the normalized densities $\rho_{0,n}^X$ converge to $\rho_0^X$ in the same weighted $L^1$ norm.
We write 
$$\frac{r_n}{c_n}-q_0=\frac{(r_n-q_0)}{c_n}+\left(\frac{1}{c_n}-1\right)q_0$$
and apply the triangle inequality to get
\begin{align*}
    &\int_{\R^d}(1+\|x\|^2)\left|\frac{r_n(x)}{c_n}-q_0(x)\right|\nu^X(x)\dx \\
    &\quad\leq \frac{1}{c_n}\int_{\R^d}(1+\|x\|^2)|r_n(x)-q_0(x)|\,\nu^X(x)\dx +\left|\frac{1}{c_n}-1\right|\int_{\R^d}(1+\|x\|^2)\rho_0^X(x)\dx \ .
\end{align*}
The first term on the right tends to zero by~\eqref{eq:ConsRegAppr:moment} and $c_n\to1$. The second tends to zero since $c_n\to1$ and $\rho_0^X$ has finite second moment.
Thus, by the weighted total variation estimate~\cite[Theorem 6.15]{villani2009optimal},
\[
    \Wass_2^2(\rho_{0,n}^X,\rho_0^X)
    \leq 2\int_{\R^d}\|x\|^2\left|\frac{r_n(x)}{c_n}-q_0(x)\right|\nu^X(x)\dx
    \longrightarrow0 \ .
\]
Finally, since $\int_{\R^d} r_n\,\nu^X \dx=c_n$, the definition of $\rho_{0,n}^X$ gives
\[
    \KL(\rho_{0,n}^X\dvert\nu^X)
    =\frac{1}{c_n}\int_{\R^d}r_n\log r_n\,\mathrm{d}\nu^X-\log c_n \to \KL(\rho_0^X\dvert\nu^X) \,,
\]
where the last limit follows from~\eqref{eq:ConsRegAppr:entropy} and $c_n\to1$.
\end{proof}

%%%%%%%%%%%%%%%%%
\section{Auxiliary results and proofs}
\label[appendix]{app:main-proof-details}

%%%%%%
\subsection{Finite second moment and \textsf{KL} divergence of the \RHMC iterates}\label[appendix]{app:main-proof-auxiliary}

\begin{lemma}\label{lem:rhmc-admissibility}
Let $\rho_0^X\in\PtwoacfsRd$ satisfy $\KL(\rho_0^X\dvert\nu^X)<\infty$. If each integration time has finite second moment, then every iterate $\rho_k^X$ generated by \RHMC belongs to $\PtwoacfsRd$ and satisfies $\KL(\rho_k^X\dvert\nu^X)<\infty$.
\end{lemma}
\begin{proof}
Recall we assume $\nu^X \propto e^{-f}$ is $L$-log-smooth for some $L \in (0,\infty)$.
We first show this implies $f$ is bounded below. Let $Z_f=\int_{\R^d}e^{-f(y)}\,\mathrm{d}y$. 
By the $L$-smoothness of $f$, for every $x,h\in\R^d$,
\[
    f(x+h)\leq f(x)+\langle\nabla f(x),h\rangle+\frac L2\|h\|^2\,.
\]
Exponentiating the negatives of both sides, integrating over $h$, and completing the square gives
\[
    Z_f\geq e^{-f(x)}\left(\frac{2\pi}{L}\right)^{d/2}
    \exp\!\left(\frac{\|\nabla f(x)\|^2}{2L}\right)
    \geq e^{-f(x)}\left(\frac{2\pi}{L}\right)^{d/2}\,.
\]
Since $Z_f<\infty$, it follows that $f_*\deq\inf_{x\in\R^d} f(x)>-\infty$.

Now suppose $\rho_{k-1}^X\in\PtwoacfsRd$ has finite \textsf{KL} divergence to $\nu^X$. 
Let $(X_t,Y_t)$ follow the Hamiltonian flow from $(X_0,Y_0)\sim\rho_{k-1}^X\otimes\gamma_{d}$, and let $\rho_{t,k}^X$ be the law of $X_t$. 
The integration time $\tau_k$ is independent of the initial condition, so $\rho_k^X=\E_{\tau_k}\left[\rho_{\tau_k,k}^X\right]$. 
For each fixed $t\geq0$, the distribution $\rho_{t,k}^X$ has a strictly positive density by~\Cref{lem:second-moment-finiteness}. The mixture therefore also admits a strictly positive density. 
Then by the convexity of \textsf{KL} divergence in its first argument and~\Cref{Lem:DescentKL},
\[
    \KL(\rho_k^X\dvert\nu^X)
    \leq\E_{\tau_k}\bigl[\KL(\rho_{\tau_k,k}^X\dvert\nu^X)\bigr]
    \leq\KL(\rho_{k-1}^X\dvert\nu^X)<\infty\,.
\]

It remains to bound the second moment. 
By conservation of energy (\Cref{lem:energy_conservation}), for all $s\geq0$,
\[
    \|Y_s\|^2=2\bigl(H(X_0,Y_0)-f(X_s)\bigr)\leq 2\bigl(H(X_0,Y_0)-f_*\bigr)\,.
\]
Therefore, from the identity $X_t-X_0=\int_0^t Y_s\ds$, we obtain $\|X_t-X_0\|^2\leq 2t^2(H(X_0,Y_0)-f_*)$. 
Applying this bound at $t=\tau_k$ and using independence, we obtain
\[
    \E\bigl[\|X_{\tau_k}\|^2\bigr]
    \leq 2\E\bigl[\|X_0\|^2\bigr]+4\E\left[\tau_k^2 \right] \cdot \E\bigl[H(X_0,Y_0)-f_*\bigr]<\infty\,.
\]
The last expectation is finite because $L$-smoothness implies $f$ has a quadratic upper bound, and $(X_0,Y_0)$ has finite second moment. 
This proves the induction step.
\end{proof}

%%%%%%%%%%%%%
\subsection{Properties of the integration time distribution}\label[appendix]{app:time-logconcave}

\begin{lemma}
\label{lem:Cw-bound-log-concave-non-inc}
Let \(w\) satisfy \cref{assump:integration-time}.
Then for all $s\geq 0$,
\begin{align*}
 \int_s^\infty (u-s)w(u)\,\du \leq \tfrac32 w(s) \,.
\end{align*}
\end{lemma}

\begin{proof}
Let $w_+(t)=\lim_{r\downarrow t}w(r)$ be the right-continuous representative of $w$. It is non-increasing and log-concave, agrees with $w$ almost everywhere, and satisfies $w_+(t)\leq w(t)$. 
It suffices to prove the estimate for $w_+$. 
In particular, $0<w_+(0)<\infty$;
this is because if $2a>0$ lies in the interior of the support, then log-concavity at the midpoint of $r$ and $2a-r$ implies for $0<r<a$:
$$w_+(r) \leq \frac{w_+(a)^2}{w_+(2a-r)} \leq \frac{w_+(a)^2}{w_+(2a)} \,,$$
where the last inequality follows since $w_+$ is non-increasing.

For $s,t\geq0$ with $w_+(s+t)>0$ and $s+t>0$, log-concavity applied to $s$ and $t$ as convex combinations of $0$ and $s+t$ implies
\[
    w_+(s)w_+(t)\geq w_+(0)w_+(s+t)\,.
\]
This also holds if $w_+(s+t)=0$ or $s+t=0$. Thus, with a change of variables $u = t + s$,
\[
    \int_s^\infty (u-s)w_+(u)\,\mathrm{d}u
    \leq \frac{w_+(s)}{w_+(0)}\int_0^\infty t w_+(t)\,\mathrm{d}t
    =\frac{w_+(s)}{w_+(0)}\E[U]\,.
\]
Thus, it remains to show $\E[U]/w_+(0) \leq \frac32$.

Let $Q$ be the nonnegative random variable with survival function $\mathbb{P}(Q>t)=w_+(t)/w_+(0)$ for $t\geq0$. 
This defines a valid survival function since $w_+$ is non-increasing, right-continuous, and tends to zero as $t \to \infty$. 
For $p=1,2,3$, integrating the survival function gives
\[
    \E[Q^p]=\frac{p}{w_+(0)}\int_0^\infty t^{p-1}w_+(t)\,\mathrm{d}t
    =\frac{p\E[U^{p-1}]}{w_+(0)}\,.
\]
Since both $q\mapsto q$ and $q\mapsto q^2$ are non-decreasing on $[0,\infty)$, $Q$ and $Q^2$ have nonnegative covariance, so $\E[Q^3]\geq\E[Q]\E[Q^2]$. 
Substituting the moment formulas gives $\E[U]/w_+(0)\leq \frac32\E[U^2]=\frac32$, as desired.
\end{proof}

%%%%%%%%%%%%%%%%%%%%%%%%%%%%%
\subsection{Proofs of the corollaries}\label[appendix]{app:corollary-proofs}

In both proofs, $\tau_k=T_kU_k$ with $U_k\sim w$ and deterministic $T_k$. 
Since $\E[U^2]=1$, Cauchy--Schwarz inequality gives $\E[U]\leq 1$, and hence
\[
    \E\left[\sum_{k=1}^K\tau_k\right]=\E[U]\sum_{k=1}^K T_k\leq\sum_{k=1}^K T_k\,.
\]
If $w$ is supported on $[0,B]$, then $\sum_{k=1}^K\tau_k\leq B\sum_{k=1}^K T_k$ almost surely. Thus, in each proof it suffices to bound $\sum_{k=1}^K T_k$.

%%%%%%%%%%%%%%%%%%%%%
\subsubsection{Proof of \Cref{cor:ConvRHMC_SLC}}\label[appendix]{app:proof-cor-ti}

\begin{proofof}{\Cref{cor:ConvRHMC_SLC}.}
Let \(q=(17\alpha+M)/(18\alpha)\) and \(\ell=\log(\KL(\rho_0^X\dvert\nu^X)/\varepsilon)>0\).
The stated choice \(K=\lceil\ell/\log(1/q)\rceil\) gives \(q^K\leq e^{-\ell}\), so \cref{thm:general-ti} gives the accuracy claim. 
Using the inequality \(-\log q\geq 1-q\),
\[
 K\leq \frac{18\alpha}{\alpha-M}\ell+1
 \leq \frac{19\alpha}{\alpha-M}\lceil\ell\rceil \,.
\]
Thus, \(\sum_{k=1}^K T_k=2K/\sqrt{\alpha-M}\) is at most
\(38\alpha\lceil\ell\rceil/(\alpha-M)^{3/2}\).
\end{proofof}

%%%%%%%%%%%%%%
\subsubsection{Proof of \Cref{cor:ConvRHMC_LC}}\label[appendix]{app:proof-cor-lc}

\begin{proofof}{\Cref{cor:ConvRHMC_LC}.}
Let $\Phi_0=\KL(\rho_0^X\dvert\nu^X)+\frac{1}{9}\Wass_2^2(\rho_0^X,\nu^X)$.
By~\Cref{thm:ConvRHMC_LC}, \(\KL(\rho_K^X\dvert\nu^X)\leq(8/9)^K\Phi_0\leq\varepsilon\).
Let \(q=\sqrt{9/8}\), so \(T_k=q^k\). 
By the definition of \(K\), we have \(q^K\leq q\sqrt{\Phi_0/\varepsilon}\), and hence
\[
\sum_{k=1}^K T_k=\frac{q(q^K-1)}{q-1}
\leq\frac{q^2}{q-1}\sqrt{\frac{\Phi_0}{\varepsilon}}
\leq19\sqrt{\frac{\Phi_0}{\varepsilon}} \,,
\]
where in the last step we have used the estimate \(q^2/(q-1)<19\).
\end{proofof}

{\small
\bibliographystyle{alpha}
\bibliography{refs}
}
\end{document}